\documentclass[10pt, conference, letterpaper]{IEEEtran}
\IEEEoverridecommandlockouts
\usepackage{cite}
\usepackage{amsmath,amssymb,amsfonts, amsthm}
\usepackage{graphicx}
\usepackage{textcomp}
\usepackage{booktabs}
\usepackage[ruled,linesnumbered]{algorithm2e}
\usepackage[noend]{algpseudocode}
\usepackage{array}
\usepackage[font=footnotesize,labelformat=simple]{subcaption}

\usepackage[table]{xcolor}
\usepackage{caption}
\usepackage{hyperref}

\newtheorem{theorem}{Theorem}
\newtheorem{assumption}{Assumption}
\newtheorem{lemma}{Lemma}
\newtheorem{proposition}{Proposition}

\usepackage{multirow}
\usepackage{bbm}
\usepackage{adjustbox} 

\def\BibTeX{{\rm B\kern-.05em{\sc i\kern-.025em b}\kern-.08em
    T\kern-.1667em\lower.7ex\hbox{E}\kern-.125emX}}
\begin{document}

\title{\vspace{-0.05in} Batching-Aware Joint Model Onloading and Offloading for Hierarchical Multi-Task Inference
\vspace{-0.1in}
}

\author{
\IEEEauthorblockN{
Seohyeon Cha\IEEEauthorrefmark{1},
Kevin Chan\IEEEauthorrefmark{2},  
Gustavo de Veciana\IEEEauthorrefmark{1},
Haris Vikalo\IEEEauthorrefmark{1}
}
\IEEEauthorblockA{\IEEEauthorrefmark{1}%
Department of Electrical and Computer Engineering, The University of Texas at Austin, Austin, TX, USA.\\
}
\IEEEauthorblockA{\IEEEauthorrefmark{2}%
DEVCOM Army Research Laboratory, Adelphi, MD, USA. 
}
\vspace{-0.3in}
}

\maketitle
\begin{abstract}
The growing demand for intelligent services on resource-constrained edge devices has spurred the development of collaborative inference systems that distribute workloads across end devices, edge servers, and the cloud. While most existing frameworks focus on single-task, single-model scenarios, many real-world applications (e.g., autonomous driving and augmented reality) require concurrent execution of diverse tasks including detection, segmentation, and depth estimation. In this work, we propose a unified framework to jointly decide which multi‑task models to deploy (``onload") at clients and edge servers, and how to route queries across the hierarchy (``offload") to maximize overall inference accuracy under memory, compute, and communication constraints. We formulate this as a mixed‑integer program and introduce J3O (Joint Optimization of Onloading and Offloading), an alternating algorithm that (i) greedily selects models to onload via Lagrangian-relaxed submodular optimization and (ii) determines optimal offloading via constrained linear programming. We further extend J3O to account for batching at the edge, maintaining scalability under heterogeneous task loads. Experiments show J3O consistently achieves over 97\% of the optimal accuracy while incurring less than 15\% of the runtime required by the optimal solver across multi-task benchmarks.
\end{abstract}


\section{Introduction}

The rapid proliferation of edge devices including smartphones, surveillance cameras, and wearables, with possible latency and privacy requirements, has sparked interest in executing Machine Learning (ML)-based inference at the edge \cite{zhou2021device}. However, as state-of-the-art ML models continue to grow in size and complexity to achieve higher accuracy, their memory and compute requirements often exceed the capabilities of resource-constrained edge hardware \cite{villalobos2022machine, dhar2024empirical}. While model compression techniques and lightweight alternatives can help reduce resource usage, they often incur an accuracy drop, particularly in multi-task settings where a compact model must simultaneously support diverse inference tasks \cite{Matsubara_2025_WACV}.

Collaborative inference systems offer a promising alternative by distributing inference workloads across hierarchical computing tiers -- end devices, edge servers, and cloud platforms \cite{fan2022dnn, fresa2022offloading, beytur2024HI}. These systems typically use lightweight models on device to handle routine or latency-sensitive queries, while selectively offloading more demanding instances to upstream servers with larger models. However, most prior frameworks have focused on single-task, single-model settings (e.g., image classification), which limits their applicability. In contrast, real-world applications such as autonomous driving, augmented reality, and smart surveillance require concurrent execution of multiple tasks (e.g., detection, segmentation, and depth estimation) within the same inference pipeline \cite{ye2022taskprompter, neseem2023adamtl, doshi2022multi}.

Extending collaborative inference to multi-task scenarios presents new challenges. A single model trained on all tasks may suffer from degraded performance due to task interference and conflicting gradients \cite{standley2020tasks}. On the other hand, maintaining a large pool of single-task models can quickly overwhelm the memory and compute capacity of edge devices. A more scalable strategy is to maintain a small library of specialized multi-task models, each trained on a subset of related tasks, and dynamically select the best model for each inference request \cite{standley2020tasks, fifty2021efficiently, song2022efficient}. While this approach has been explored in centralized or cloud-based settings, model selection and deployment under tight system constraints in hierarchical, distributed environments remains largely unaddressed.

The joint model placement and inference offloading has previously been studied primarily in single-task settings, and has considered either a single model or a collection of models. Early efforts in edge intelligence~\cite{li2018zalad, kang2017neurosurgeon, li2019edge, he2020jpdra} explored collaborative execution of a single Deep Neural Network (DNN) across edge and cloud, using techniques such as DNN partitioning and early exiting to reduce end-to-end latency by balancing communication and computation. More recent work has extended this to multi-model scenarios, where systems must select, cache, and place multiple candidate models across heterogeneous edge resources \cite{hudson2021qos, sada2024energy, zhang2021deep, chai2024joint, chen2022online}. These approaches aim to optimize inference accuracy, latency, or energy under resource constraints, typically by coordinating model selection and query routing between edge and cloud. However, they remain largely restricted to single-task inference and shallow hierarchies, without jointly optimizing model placement and task routing across clients, edge and cloud. Table~\ref{tab:my_label} summarizes the distinctions between our approach and these prior methods.

\begin{figure*}[t]
    \centering
    \begin{subfigure}[t]{0.24\linewidth}
        \centering
        \includegraphics[width=1.02\linewidth]{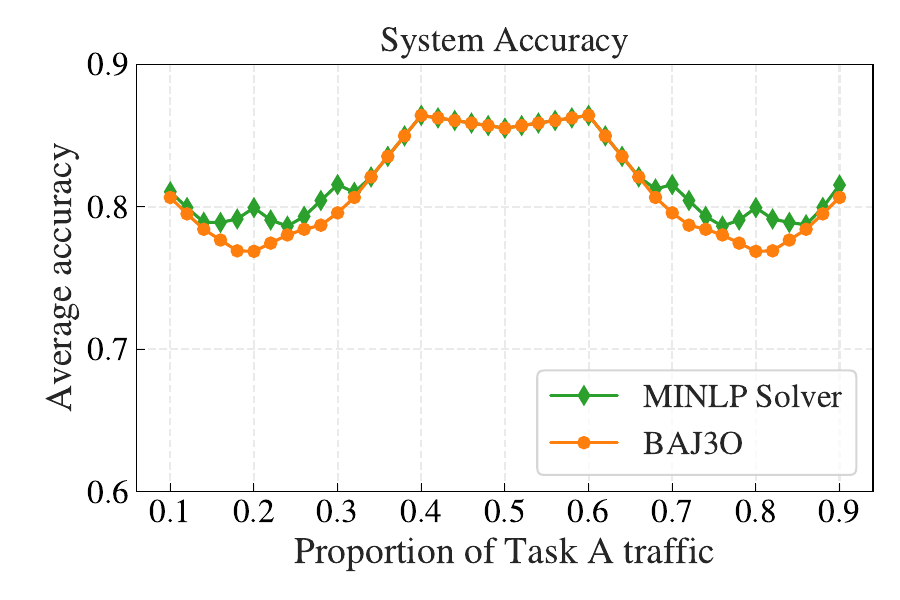}
        \label{}
        \vspace{-0.21in}
    \end{subfigure}
    \hfill
    \begin{subfigure}[t]{0.24\linewidth}
        \centering
        \includegraphics[width=1.02\linewidth]{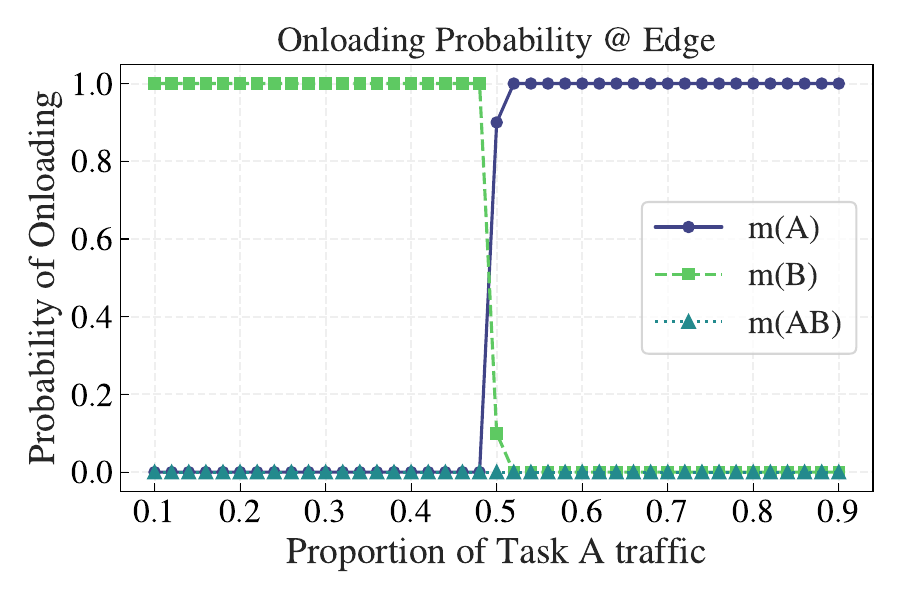}
        \label{}
        \vspace{-0.21in}
    \end{subfigure}
    \hfill
    \begin{subfigure}[t]{0.24\linewidth}
        \centering
        \includegraphics[width=\linewidth]{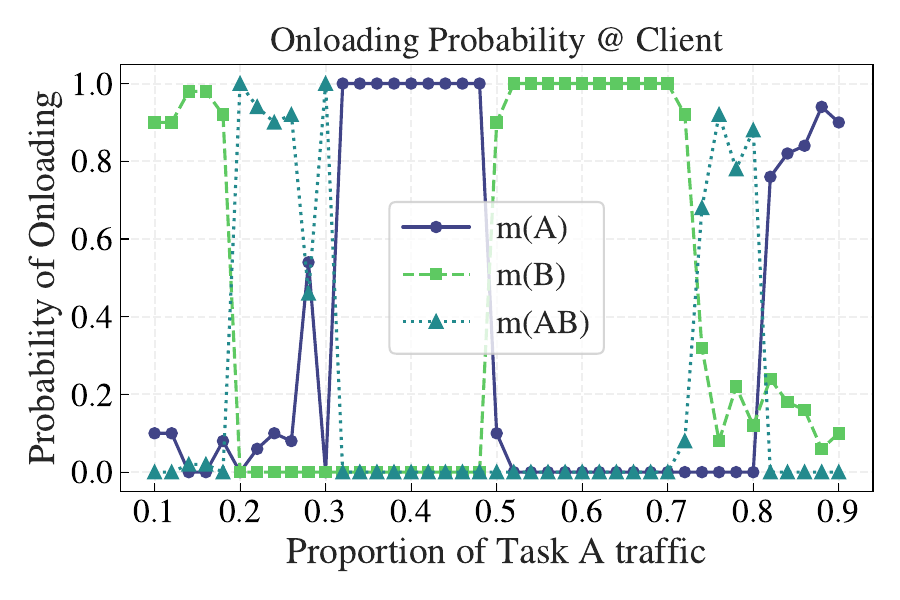}
        \label{fig:dom-loss}
        \vspace{-0.21in}
    \end{subfigure}
    \hfill
    \begin{subfigure}[t]{0.24\linewidth}
        \centering
        \includegraphics[width=1.02\linewidth]{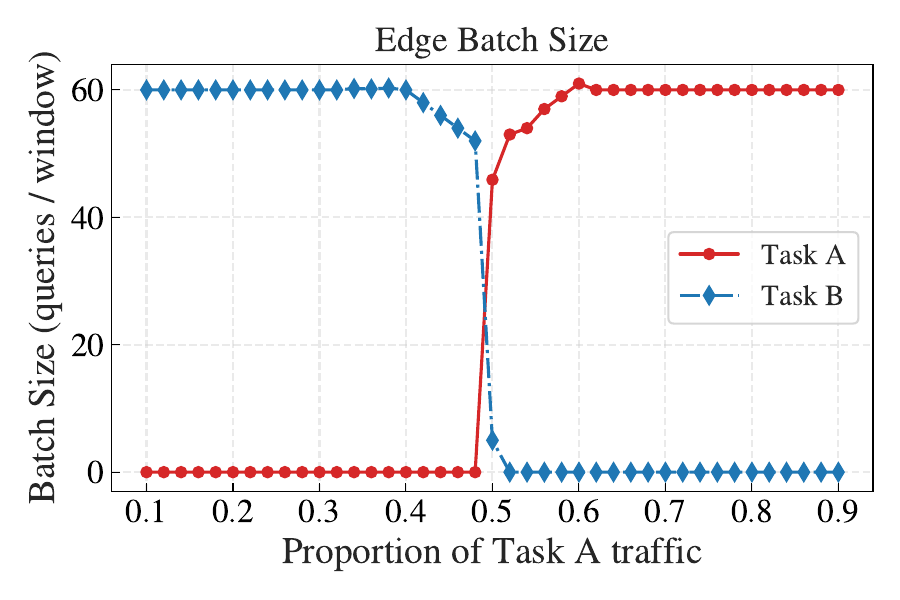}
        \label{fig:dom-loss}
        \vspace{-0.21in}
    \end{subfigure}
    \begin{subfigure}[t]{0.24\linewidth}
        \centering
        \includegraphics[width=1.02\linewidth]{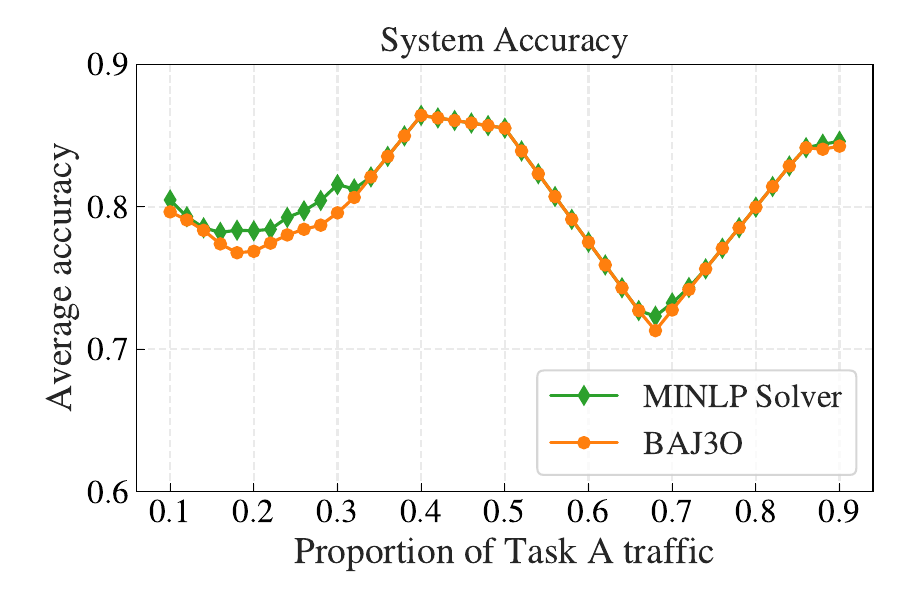}
        \label{}
        \vspace{-0.21in}
        \caption{Accuracy}
    \end{subfigure}
    \hfill
    \begin{subfigure}[t]{0.24\linewidth}
        \centering
        \includegraphics[width=1.02\linewidth]{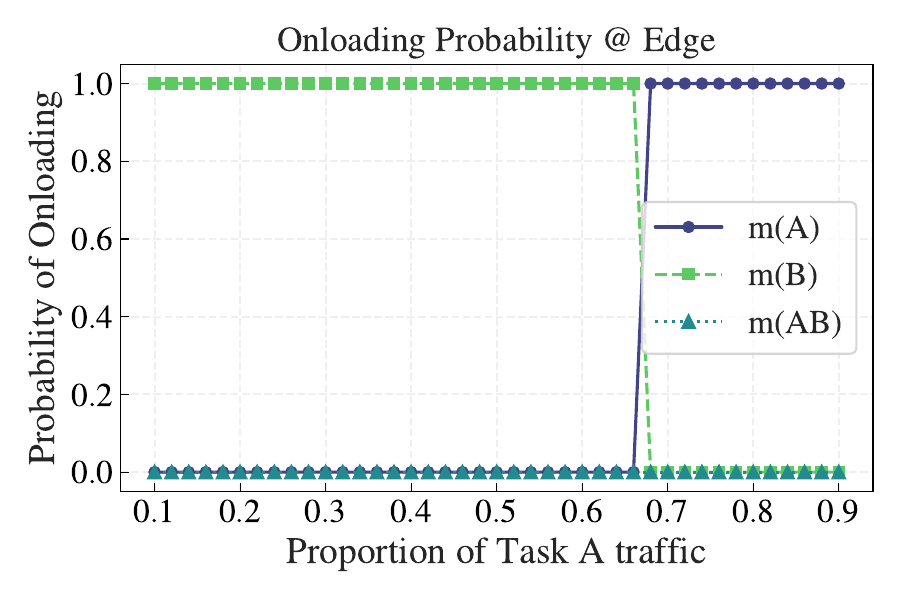}
        \label{}
        \vspace{-0.21in}
        \caption{Edge Onloading Decision}
    \end{subfigure}
    \hfill
    \begin{subfigure}[t]{0.24\linewidth}
        \centering
        \includegraphics[width=1.02\linewidth]{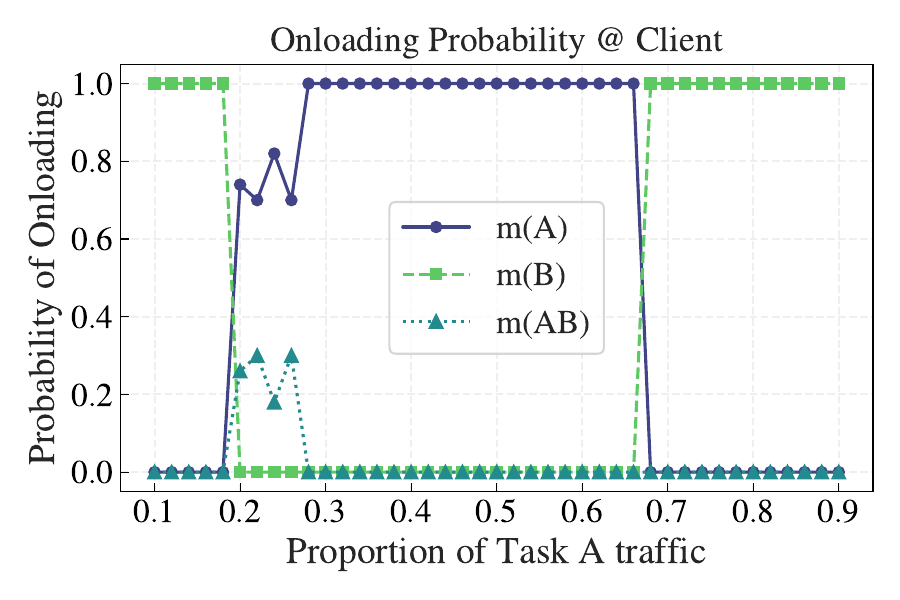}
        \label{fig:dom-loss}
        \vspace{-0.21in}
        \caption{Client Onloading Decision}
    \end{subfigure}
    \hfill
    \begin{subfigure}[t]{0.24\linewidth}
        \centering
        \includegraphics[width=1.02\linewidth]{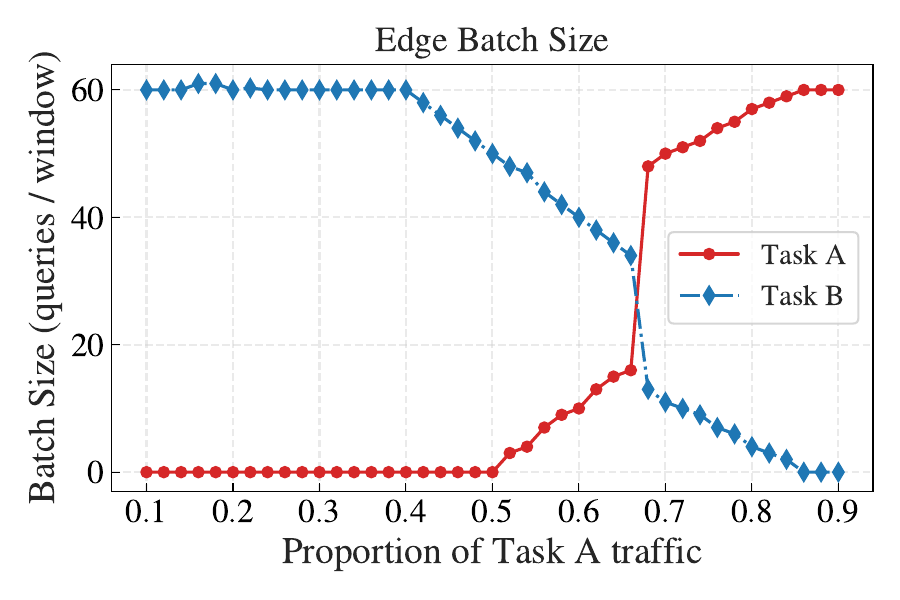}
        \label{fig:dom-loss}
        \vspace{-0.21in}
        \caption{Batch Size}
    \end{subfigure}
    \caption{We compare (top) homogeneous and (bottom) heterogeneous task load distributions across clients. As task A becomes dominant, the optimizer adapts model placement and offloading strategies to maximize batching efficiency and system accuracy.}
    \label{fig:impact-load}
\vspace{-0.18in }
\end{figure*}

\begin{table}[t]
    \centering
    \footnotesize
    \caption{Comparison of related works by objective, model granularity, task scope, and hierarchy of models for inference.}
    \vspace{-0.075in}
    \begin{tabular}{c|ccccc}
        \bottomrule
         & \textbf{Objective} & \textbf{Model} & \textbf{Task} & \textbf{Hierarchy}\\
        \hline
        \cite{hudson2021qos} & Accuracy \& Latency & Multi & Single & Single-Tier\\
        \cite{chai2024joint} & Latency \& Energy & Multi & Single & Single-Tier\\
        \cite{chen2022online} & Latency \& Energy & Multi & Single & Two-Tier \\
        \cite{sada2024energy, zhang2021deep} & Accuracy & Multi & Single  & Two-Tier \\ 
        \cite{wu2024share} & Latency & Multi & Multi & Single-Tier \\
        \hline
        Ours & Accuracy & Multi & Multi & Multi-Tier \\ 
         \toprule
    \end{tabular}
    \label{tab:my_label}
    \vspace{-0.25in}
\end{table}
A recent study~\cite{wu2024share} explores joint multi-task model deployment and task offloading in vehicular edge computing, where vehicles send inference queries to roadside units (RSUs) equipped with shared multi-task models. However, their ``one-model-fits-all" approach limits task specialization and may underperform compared to frameworks that rely on tailored models. The framework also lacks hierarchical offloading beyond RSUs, leaving edge device and cloud resources unused. The broader problem of jointly optimizing multi-task model placement and hierarchical task routing under system constraints remains largely unexplored.

To address these challenges, we propose a unified framework for joint model onloading and hierarchical offloading in distributed multi-task inference systems. Each client hosts a tailored subset of multi-task models and selectively offloads queries to higher-tier edge servers. In turn, edge servers may forward difficult requests to the cloud, which stores full-capacity models. The framework jointly determines (i) which models to onload at the client and edge levels and (ii) how to route queries across system tiers to maximize overall accuracy under memory, compute, and communication constraints. To efficiently solve this combinatorial problem, we introduce the \textbf{J}oint \textbf{O}ptimization of \textbf{O}nloading and \textbf{O}ffloading (\textbf{J3O}) algorithm, an alternating method that combines greedy submodular maximization for model onloading with constrained linear programming for task offloading.

We further extend our formulation to incorporate batching, allowing edge servers to aggregate queries into homogeneous batches to improve GPU utilization. This design aligns with modern accelerator architectures, where batching amortizes fixed inference overhead. In order to bound latency, we impose a time-window constraint on batch formation. To solve the extended problem, we develop \textbf{BAJ3O} (\textbf{B}atching-\textbf{A}ware \textbf{J}oint \textbf{O}ptimization of \textbf{O}nloading, \textbf{O}ffloading), which augments J3O with a batching-related latency constraint and integrates it seamlessly into the alternating optimization loop.

The main contributions of this work are:
\begin{itemize}
    \item We formulate a new joint model onloading and offloading problem for hierarchical multi-task inference, aiming to maximize system-wide accuracy under memory, compute, and communication constraints. This formulation captures task-aware model specialization, multi-tier coordination, and shared resource budgets.
    \item We propose an alternating optimization algorithm with provable performance guarantees, \textbf{J3O} (\textbf{J}oint \textbf{O}ptimization of \textbf{O}nloading and \textbf{O}ffloading), which decomposes the mixed-integer nonlinear program into two tractable subproblems: greedy submodular model selection and LP-based task offloading.
    \item We extend the formulation to incorporate batching for efficient GPU utilization. The resulting \textbf{BAJ3O} algorithm (\textbf{B}atching-\textbf{A}ware \textbf{J}oint \textbf{O}ptimization of \textbf{O}nloading, \textbf{O}ffloading) introduces a linear surrogate for batching constraints and integrates it into the J3O framework.
    \item We evaluate our approach on standard multi-task benchmarks and show consistent improvements in the accuracy–runtime trade-off over the baselines.
\end{itemize}

\subsection{Motivating Example}
To illustrate how joint model onloading and task offloading interact under system constraints, we present a controlled experiment examining two key factors: (i) the effect of task imbalance and traffic distribution, and (ii) the impact of batching overheads on system behavior. We consider a highly simplified two-tier architecture with one edge server and five clients. Each client and the edge node may onload a single model, and the system supports two tasks: A and B.

Three multi-task models are available: (1) $m(A)$, achieving 90\% accuracy on task A, 10\% on task B; (2) $m(B)$, 90\% on A, 10\% on B; and (3) $m(AB)$, 60\% on both tasks. Client-side models are compressed and achieve 90\% of their original accuracy. The edge can host any of the three models at full accuracy. The overall system is constrained by a shared communication bandwidth across clients, per-client compute budgets, and latency-bounded batching at the edge.

\paragraph{Impact of Multi-task Traffic}
We first explore how task imbalance affects model deployment. In this setting, the system faces tight communication and client-compute constraints. We vary the global fraction of task A traffic $p_A$ from 0.1 to 0.9 while keeping total query volume fixed.

\textbf{Balanced vs. skewed task loads.}
Fig.~\ref{fig:impact-load} (top) shows the outcome under a homogeneous client-task distribution, i.e., where all clients observe the same fraction of A and B queries.  When global task traffic is balanced ($p_A\!=\!0.5$), clients offload both tasks to the edge, forming batches of both tasks. As task A becomes dominant, the edge allocates most capacity to A. Once global task A loads exceed the offloading budget, clients begin to onload models to handle residual traffic. Initially, they adopt task-specific models; as traffic becomes more mixed, 
$m(AB)$ is selected instead. Batch sizes remain stable, but the onloaded model mix adapts to the evolving load.

Fig.~\ref{fig:impact-load} (bottom) shows the heterogeneous case, where task B workload is concentrated on a single ``hot" client (70\% of its queries), while task A remains uniformly distributed. The system continues to prioritize edge offloading of task B due to batching benefits for the hot client. This shifts the transition threshold at which onloading becomes necessary: edge onloading to 
$m(B)$ persists even when global traffic appears balanced, delaying the need to onload $m(B)$ at clients. When client constraints tighten, the system begins using suboptimal models locally, slightly reducing overall accuracy.

\begin{figure}[t]
    \centering
    \begin{subfigure}[t]{0.493\linewidth}
        \centering
        \includegraphics[width=1.02\linewidth]{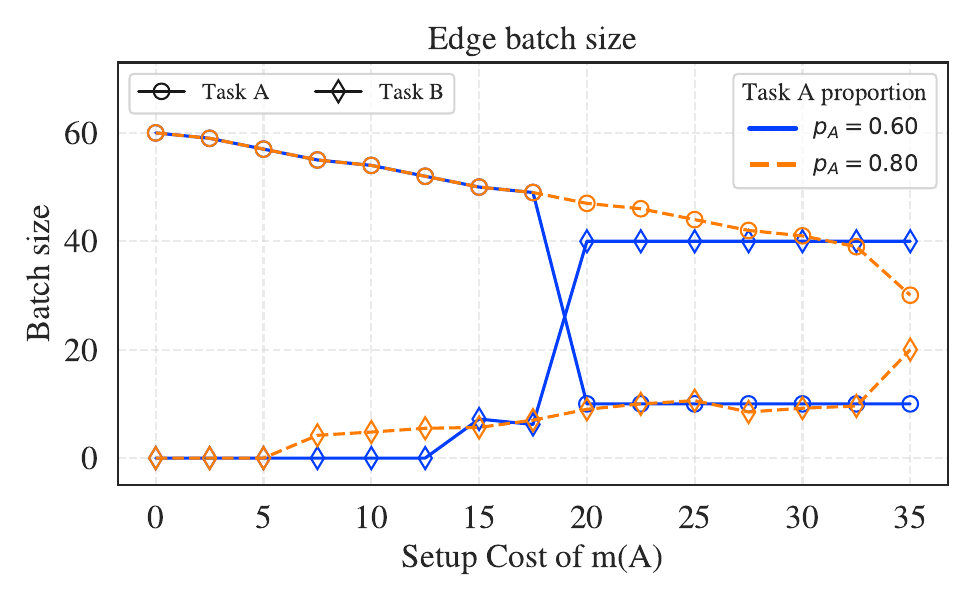}
        \label{}
        \vspace{-0.21in}
    \end{subfigure}
    \hfill
    \begin{subfigure}[t]{0.493\linewidth}
        \centering
        \includegraphics[width=1.02\linewidth]{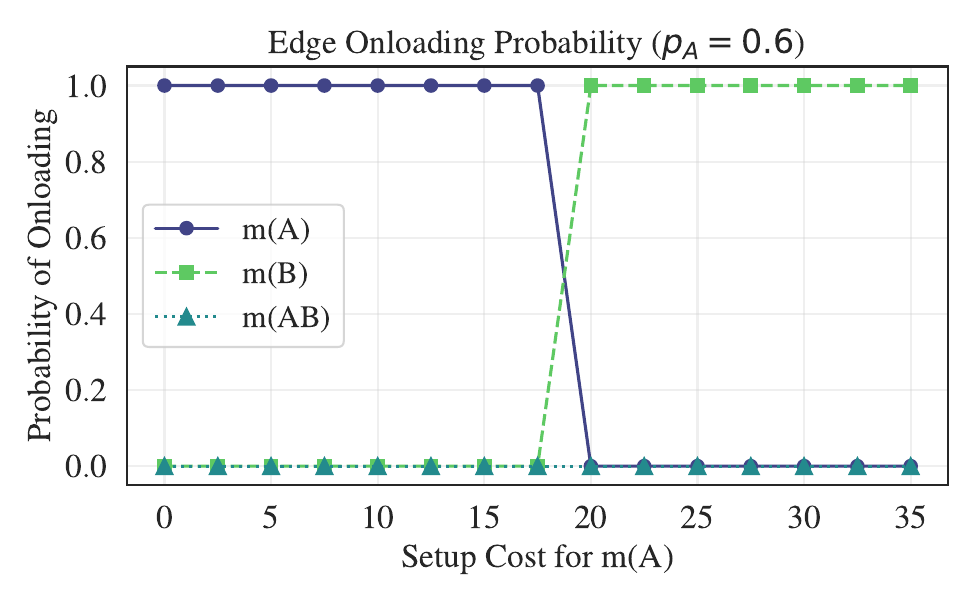}
        \label{}
        \vspace{-0.21in}
    \end{subfigure}
    \hfill
    \vspace{-0.21in}
    \caption{Higher setup cost for $m(A)$ prompts the system to shift to alternative models or local execution.}
    \label{fig:impact-setup}
\vspace{-0.2 in }
\end{figure}

\paragraph{Impact of Batching}
Next, we consider the impact of batching. Each edge model has a fixed launch cost per batch, better amortized with larger batches, but limiting queries that can be served within a latency bound.
We examine how increasing this setup cost influences the optimizer's decisions.

\textbf{Higher setup cost reduces batching efficiency.}
Fig.~\ref{fig:impact-setup} shows the impact of gradually increasing the setup cost for 
$m(A)$. As the cost rises, batches shrink, reducing the benefit of offloading A. The system compensates by switching to alternative models: at task A proportion $p_A=0.6$, it adopts $m(B)$ at the edge and relies on client-side model $m(A)$ to handle A tasks. This shows how elevated setup costs can alter onloading strategies to optimize overall system efficiency. 

These behaviors highlight the complex interplay between model accuracy, batching efficiency, and system constraints. Even in this simplified setting, optimal decisions emerge from a careful balance of onloading, offloading, and batching, motivating the need for a principled optimization framework which we develop in the next sections.



\section{System Model}
We consider a three-tier inference system comprising clients $\mathcal{C} = \{1, \dots, C\}$, edge servers $\mathcal{E} = \{1, \dots, E\}$, and a centralized cloud. Each client $c$ is connected to a designated edge server $e$, represented by a binary variable $y^{c,e} \in \{0,1\}$, and all edge servers connect to the cloud. The system supports a set of ML tasks $\mathcal{T} = \{1, \dots, T\}$ using a shared library of multi-task models $\mathcal{M} = \{1, \dots, M\}$, each offering distinct accuracy-cost trade-offs. Clients generate streams of multi-task inference queries and may \emph{onload} a subset of models locally, subject to resource constraints. Queries can also be \emph{offloaded} to the assigned edge server, which can either serve them using its own onloaded models or forward them to the cloud for maximum accuracy. This yields a hierarchical inference strategy, where model placement and task routing decisions jointly determine whether inference occurs at the client, edge, or cloud. Our goal is to maximize the system-wide average inference accuracy, weighted by task load.


\subsection{Inference Model}


Each client $c \in \mathcal{C}$ generates a stream of inference queries, each corresponding to a task $t \in \mathcal{T}$. These tasks may include workloads such as edge detection, semantic segmentation, or domain-specific classification. The long-term task demand at client $c$ is represented by a vector $\boldsymbol{\lambda}^c = (\lambda_t^c)_{t=1}^T$, where $\lambda_t^c$ denotes the average arrival rate of queries for task $t$, measured in jobs per unit time.

To support client workloads, a library of pre-trained models $\mathcal{M}$ is stored at the cloud. Each model $m \in \mathcal{M}$ may support one or more tasks. Clients and edge servers must select which models to \emph{onload} locally, subject to their resource budgets. We define binary variables $x_m^c, x_m^e \in \{0,1\}$ to indicate whether model $m$ is onloaded at client $c$ or edge server $e$, respectively. Given a set of onloaded models, each task query must be served by exactly one model at each tier. To capture this, we introduce binary selection variables $z_{m,t}^c$ and $z_{m,t}^e$, which indicate whether model $m$ is used to serve task $t$ at client $c$ or edge server $e$, respectively. For convenience, we define the model onloading vectors $\mathbf{x}^c = (x_m^c)_{m \in \mathcal{M}}$ and $\mathbf{x}^e = (x_m^e)_{m \in \mathcal{M}}$, and the task-model selection vectors $\mathbf{z}^c = (z_{m,t}^c)_{m \in \mathcal{M},\; t \in \mathcal{T}}$ and $\mathbf{z}^e = (z_{m,t}^e)_{m \in \mathcal{M},\; t \in \mathcal{T}}$.

Offloading decisions govern how task queries are routed through the system hierarchy. We define continuous variables $o_t^c, o_t^{c,e} \in [0,1]$, where $o_t^c$ denotes the fraction of task-$t$ queries from client $c$ that are offloaded to its assigned edge server, and $o_t^{c,e}$ denotes the fraction of those edge-level queries that are further offloaded to the cloud. Setting $o_t^c = 0$ corresponds to full local execution at the client, while $o_t^c = 1$ means all task-$t$ queries are sent to the edge. Similarly, $o_t^{c,e} = 1$ implies full offloading from edge to cloud. For convenience, we denote the client- and edge-level offloading vector as $\mathbf{o}^c = (o_t^c)_{t \in \mathcal{T}}$ and $\mathbf{o}^{c,e} = (o_t^{c,e})_{t \in \mathcal{T}}$, respectively.

Each client $c$ is associated with exactly one edge server $e$, indicated by a binary variable $y^{c,e} \in \{0,1\}$, such that $\sum_{e \in \mathcal{E}} y^{c,e} = 1$. Clients connected to the same edge server share communication bandwidth, and all edge servers share a common communication link to the cloud. As we describe in the following section, both model onloading and task offloading decisions must respect memory, compute, and communication constraints across all tiers.

\subsection{Accuracy Modeling}

We define the system-wide inference accuracy as the average task accuracy, weighted by the long-term task demands across all clients. Let $\mathbf{x} = ((\mathbf{x}^c)_{c}, (\mathbf{x}^e)_{e})$ denote the model onloading decisions at the client and edge levels, and let $\mathbf{z} = ((\mathbf z^c)_c, (\mathbf z^e)_e)$ capture the selected inference model for each task at each client and edge server. To ensure that only onloaded models are used for inference, we enforce that
\vspace{-0.04in}
\begin{subequations}
\begin{align}
    & z_{m,t}^c \le x_m^c, \quad \sum\nolimits_{m \in \mathcal{M}}\,z_{m,t}^c = 1, \;\forall m,t,c, \label{eq:assign-c}\\
    & z_{m,t}^e \le x_m^e, \quad \sum\nolimits_{m \in \mathcal{M}} \,z_{m,t}^e = 1, \;\forall m,t,e. \label{eq:assign-e}
\end{align}
\label{const:assign-x-z}
\end{subequations}
Let $\mathbf{o}\!=\!((\mathbf o^c)_c, (\mathbf o^{c,e})_{c,e})$ denote the offloading decisions from clients to edges and from edges to the cloud. Define the total system demand as $\lambda_\text{tot}\!:=\! \sum_{t\in\mathcal T}\sum_{c\in\mathcal C}\lambda_t^c$, and the normalized task demand for client $c$ as $\bar{\lambda}_t^c\!:=\!\lambda_t^c/\lambda_\text{tot}$. The overall system-wide accuracy is then
\vspace{-0.04in}
\begin{equation}
    \begin{aligned}
    F(\mathbf{x}, \mathbf{z}&, \mathbf{o}) = \sum_{c\in\mathcal C} F_\text{client}(\mathbf x^c, \mathbf z^c, \mathbf o^c) \\
    &+ \sum_{e\in \mathcal E}\sum_{c \in \mathcal C_e} [F_\text{edge}(\mathbf x^e, \mathbf z^e, \mathbf o^c, \mathbf o^{c,e}) + F_\text{cloud}(\mathbf o^{c,e})], \label{eq:main_obj}
    \end{aligned}
\end{equation}
where $\mathcal{C}_e := \{c\in\mathcal C\!\mid\! y^{c,e} = 1\}$ is the set of clients assigned to edge server $e$. This objective captures average task accuracy across client, edge, and cloud levels:

\textbf{(1) Client-side accuracy} reflects the accuracy of locally served queries,
\vspace{-0.05in}
\begin{align}
F_\text{client}(\mathbf{x}^c, \mathbf{z}^c, \mathbf{o}^c) = \sum_{t \in \mathcal{T}} \bar{\lambda}_t^c (1 - o_t^c) A_t^c(\mathbf{x}^c, \mathbf{z}^c),
\end{align}
where $A_t^c(\cdot)$ denotes the accuracy for task $t$ under client-side execution,
\vspace{-0.08in}
\begin{equation}
A_t^c(\mathbf{x}^c, \mathbf{z}^c) = \sum_{m \in \mathcal{M}} \sum_{e \in \mathcal{E}} y^{c,e} a_{m,t}^e x_m^c z_{m,t}^c.
\end{equation}
Here, $y^{c,e}$ indicates the client's edge assignment, and $a_{m,t}^e$ denotes the accuracy of model $m$ on task $t$ under edge-$e$ data distribution. This captures location-dependent accuracy variations in mobile settings. Tasks unsupported by a model are assigned zero accuracy. 

\textbf{(2) Edge-side accuracy} accounts for queries handled at the edge but not forwarded to the cloud, 
\vspace{-0.02in}
\begin{align}
\!\!F_\text{edge}(\mathbf x^e, \mathbf z^e, \mathbf o^c, \mathbf o^{c,e}) 
= {\sum\nolimits}_{t\in\mathcal T}\,\bar\lambda_t^{c} (o_t^c - o_t^{c,e}) A_t^e(\mathbf x^e, \mathbf z^e),
\end{align} 
where
\vspace{-0.1in}
\begin{equation}
A_t^e(\mathbf{x}^e, \mathbf{z}^e) = {\sum\nolimits}_{m \in \mathcal{M}}\,a_{m,t}^e\,x_m^e z_{m,t}^e
\end{equation}
denotes the accuracy of edge-level inference.

\textbf{(3) Cloud-side accuracy} captures queries ultimately served by the cloud,
\vspace{-0.1in}
\begin{align}
F_\text{cloud}(\mathbf o^{c,e}) = {\sum\nolimits}_{t\in\mathcal T}\,\,\bar\lambda_t^c o_t^{c,e} A_t^s,
\end{align}
where $A_t^s$ is the oracle-level accuracy for task $t$ at the cloud.

\subsection{Resource Constraints}
\subsubsection{Memory and Computation Constraints}
Clients and edge servers have limited memory capacity, denoted by $\mu^c$ and $\mu^e$, respectively. Each model $m \in \mathcal{M}$ requires $s_m$ bytes of memory. Onloaded models must fit within the available memory, i.e.,
\vspace{-0.03in}
\begin{subequations} \label{const:mem}
\begin{align}
{\sum\nolimits}_{m \in \mathcal{M}}\, x_m^c s_m \leq \mu^c, \label{const:mem-c} \\
{\sum\nolimits}_{m \in \mathcal{M}}\, x_m^e s_m \leq \mu^e. \label{const:mem-e}
\end{align}
\end{subequations}
We consider only model parameter memory, assuming input data is small relative to model size.

Inference also consumes compute resources. Let $w_m$ denote the per-query compute cost of model $m$. For client $c$, the expected compute load, weighted by task arrival rate $\lambda_t^c$ and local execution ratio $1 - o_t^c$, must not exceed its compute capacity $\beta^c$, i.e.,
\begin{equation}
\sum_{m \in \mathcal{M}} \sum_{t \in \mathcal{T}} \lambda_t^c (1 - o_t^c) \cdot w_m \cdot x_m^c z_{m,t}^c \le \beta^c.
\label{const:comp-c}
\end{equation}
Similarly, edge server $e$ must serve its assigned load within compute budget $\beta^e$, that is
\begin{equation}
\sum_{m \in \mathcal{M}} \sum_{t \in \mathcal{T}} \sum_{c \in \mathcal{C}_e} \lambda_t^c (o_t^c - o_t^{c, e}) \cdot w_m \cdot x_m^e z_{m,t}^e \le \beta^e.
\label{const:comp-e}
\end{equation}
We assume the cloud has sufficient capacity to store the entire model library and handle any forwarded queries.

\subsubsection{Shared Offloading Constraint}
Communication bandwidth is shared across multiple system tiers. Let $d_t$ denote the input data size for task $t$. For each edge server $e \in \mathcal{E}$, the total data offloaded from its assigned clients must not exceed its uplink bandwidth budget $\kappa^e$ (in bytes per unit time),
\begin{equation}
\sum_{t \in \mathcal{T}} \sum_{c \in \mathcal{C}_e} \lambda_t^c o_t^c  d_t \leq \kappa^e.
\label{const:comm-c}
\end{equation}
Similarly, the cloud server is subject to a global communication constraint $\kappa^s$. The total traffic forwarded from all edge servers must satisfy
\begin{equation}
\sum_{t \in \mathcal{T}} \sum_{e \in \mathcal{E}} \sum_{c \in \mathcal{C}_e} \lambda_t^c o_t^{c,e} d_t \leq \kappa^s.
\label{const:comm-e}
\end{equation}

These constraints regulate network-level data transmission and implicitly bound system latency. As bandwidth approaches saturation, queuing and congestion may increase, causing delay. Limiting offloading volumes helps maintain timely task execution by preventing overload-induced latency.

\section{Joint Onloading and Offloading}
\subsection{Problem Formulation}
We formulate the joint model onloading and offloading problem as the maximization of system-wide average inference accuracy across all tasks,
\vspace{-0.05in}
\begin{subequations}
\begin{align}
    P1:  \max_{\mathbf x, \mathbf z, \mathbf o} & \;F(\mathbf x, \mathbf z, \mathbf o) \label{eq:joint_prob}\\
    \text{s.t. }  & \eqref{const:assign-x-z}, \eqref{const:mem}-\eqref{const:comm-e}, \\
    & o_t^{c,e} \le o_t^c \, y^{c,e}, \quad \forall t, c, e\label{const:o_t_ce_assign} \\
    & x_m^c, z_{m,t}^c \in \{0, 1\}, \quad \forall m,t,c\label{const:x_z_c_range} \\ 
    & x_m^e, z_{m,t}^e \in \{0, 1\}, \quad\forall m,t,e\label{const:x_z_e_range} \\
    & o_t^c, o_t^{c,e} \in [0,1], \quad\forall t,c,e. \label{const:o_range}
\end{align}
\label{eq:joint_opt}
\end{subequations}
Here, the objective $F(\cdot)$ captures the accuracy contributions from client-, edge-, and cloud-level inference as previously defined. The constraints enforce valid model selection (\ref{const:assign-x-z}), resource limits (\ref{const:mem}–\ref{const:comm-e}), hierarchical offloading consistency (\ref{const:o_t_ce_assign}), and variable domains (\ref{const:x_z_c_range}–\ref{const:o_range}).

This is a \textit{mixed-integer nonlinear program} (MINLP). Even a simplified version reduces to the 0-1 knapsack problem, which is NP-hard. Hence, the full joint onloading–offloading problem is also NP-hard.

\subsection{Problem Decomposition}

To address the complexity of the original MINLP, we decompose it into two tractable subproblems -- model onloading and inference offloading. The onloading selects models to deploy and forms a constrained submodular maximization. The offloading assigns inference queries across the hierarchy and reduces to a constrained linear program.

\vspace{0.05in}
\noindent\textbf{Model Onloading Subproblem.} Given fixed offloading decisions $\mathbf{o}$, model onloading decomposes across clients and edges, each forming a constrained submodular maximization. For client $c$, the problem is
\begin{align}
& \max_{\mathbf{x}^c, \mathbf{z}^c}
F_{\text{client}}(\mathbf{x}^c, \mathbf{z}^c) := {\sum\nolimits}_{t \in \mathcal{T}} \bar\lambda_t^c (1 - o_t^c) A_t^c(\mathbf{x}^c, \mathbf{z}^c). \label{eq:onload_prob_c} \\
& \quad \text{s.t.} \quad \eqref{eq:assign-c}, \eqref{const:mem-c}, \eqref{const:comp-c}, \eqref{const:x_z_c_range}. \notag
\end{align}
This maximizes average client-side accuracy subject to memory, compute, and assignment constraints. Let $\mathcal{M}^c \subseteq \mathcal{M}$ be the models selected by client $c$. The objective simplifies to
\begin{equation}
f(\mathcal{M}^c) = {\sum\nolimits}_{t \in \mathcal{T}} \bar\lambda_t^c (1 - o_t^c) \max_{m \in \mathcal{M}^c} a_{m,t}^e,
\end{equation}
where $a_{m,t}^e$ captures the model's task accuracy under the data distribution at edge $e$, to which client $c$ is assigned.

\begin{proposition}
For a fixed offloading decision $\mathbf o^c$, the client-side objective $f(\mathcal{M}^c)$ is monotone and submodular.
\end{proposition}

\begin{proof}
Monotonicity follows directly -- adding a model to $\mathcal{M}^c$ cannot reduce the maximum accuracy for any task.

To show submodularity, let $\mathcal{S} \subseteq \mathcal{T} \subseteq \mathcal{M}^c$ and $m \in \mathcal{M}^c \setminus \mathcal{T}$. For each task $t \in \mathcal{T}$, define $a_{\mathcal{S}}(t) := \max_{m' \in \mathcal{S}} a_{m',t}^e$ and $a_{\mathcal{T}}(t) := \max_{m' \in \mathcal{T}} a_{m',t}^e$, so $a_{\mathcal{S}}(t) \le a_{\mathcal{T}}(t)$. The marginal gain of adding $m$ to set $\mathcal S$ is $\Delta_f(t; m \mid \mathcal{S}) = \bar\lambda_t^c (1 - o_t^c) \left( \max\{a_{\mathcal{S}}(t), a_{m,t}^e\} - a_{\mathcal{S}}(t) \right)$. Hence, by submodularity, the marginal gain must exhibit diminishing returns, i.e., $\Delta_f\!\bigl(t; m \mid \mathcal{S}\bigr)\;\ge\;
\Delta_f\!\bigl(t; m \mid \mathcal{T}\bigr)$, which we verify as follows:
\begin{itemize}
\item[(1)] If $a_{m,t}^e \le a_{\mathcal{S}}(t)$, then the marginal gain is zero for both sets, i.e., $\Delta_f(t; m \mid \mathcal{S}) = \Delta_f(t; m \mid \mathcal{T}) = 0$. 
\item[(2)] If $a_{m,t}^e > a_{\mathcal{S}}(t)$, then $\Delta_f(t; m \mid \mathcal{S}) = \bar\lambda_t^c (1 - o_t^c) \left( a_{m,t}^e - a_{\mathcal{S}}(t) \right)$ and 
$\Delta_f(t; m \mid \mathcal{T}) = \bar\lambda_t^c (1 - o_t^c) \left( \max\{a_{m,t}^e,  a_{\mathcal{T}}(t)\} - a_{\mathcal{T}}(t) \right)$. Since $a_{\mathcal{S}}(t) \le a_{\mathcal{T}}(t)$, the difference $a_{m,t}^e - a_{\mathcal{S}}(t)$ is at least as large as $\max\{a_{m,t}^e, a_{\mathcal{T}}(t)\} - a_{\mathcal{T}}(t)$, proving the inequality.
\end{itemize}
Summing over $t$ confirms submodularity of $f(\mathcal{M}^c)$.
\end{proof}

Next, define the effective task load handled at edge server $e$ as $\lambda_t^e := \sum_{c \in \mathcal{C}e} \lambda_t^c (o_t^c - o_t^{c,e})$ and normalize it as $\bar\lambda_t^e = \lambda_t^e/\lambda_\text{tot}$. Since the cloud-side accuracy is fixed under given offloading rates, the edge-side subproblem becomes
\vspace{-0.05in}
\begin{align}
& \max_{\mathbf{x}^e, \mathbf{z}^e} 
F_{\text{edge}}(\mathbf{x}^e, \mathbf{z}^e) :={\sum\nolimits}_{t \in \mathcal{T}} \,\,\bar\lambda_t^e A_t^e(\mathbf{x}^e, \mathbf{z}^e)  \label{eq:onload_prob_e}\\
& \quad \text{s.t.} \quad
\eqref{eq:assign-e}, \eqref{const:mem-e}, \eqref{const:comp-e}, \eqref{const:x_z_e_range}. \notag
\vspace{-0.05in}
\end{align}
As with the client-side case, this objective defines a weighted coverage function over tasks, and is thus monotone and submodular in the set of selected model-task pairs. With fixed offloading, the full objective in \eqref{eq:main_obj} becomes the sum of two monotone submodular functions (i.e., client-side and edge-side terms) plus a constant cloud-side term $F_\text{cloud}(\cdot)$. Hence, the overall model onloading problem reduces to a monotone submodular maximization.

\vspace{0.05in}
\noindent\textbf{Offloading Subproblem}
Given fixed model selections $\mathbf{x}$ and task assignments $\mathbf{z}$ from the onloading stage, the offloading subproblem determines the optimal offloading rates $\mathbf{o}$ to maximize the average system accuracy while satisfying system-wide compute and communication constraints. The resulting formulation is a constrained linear program
\begin{align}
\max_{\mathbf{o}} \, &
\sum_{t \in \mathcal{T}} \sum_{e \in \mathcal E} \sum_{c \in \mathcal{C}_e} \bar\lambda_t^c [
(1 - o_t^c) A_t^c(\mathbf{x}^c, \mathbf{z}^c)\notag \\
& \qquad + (o_t^c - o_t^{c,e}) A_t^e(\mathbf{x}^e, \mathbf{z}^e)+ o_t^{c,e} A_t^s] \label{eq:offload-obj}\\
\text{s.t. } & \quad \eqref{const:comp-c}-\eqref{const:comm-e}, \eqref{const:o_t_ce_assign}, \eqref{const:o_range}. \notag
\end{align}

\section{Algorithm Design}
\subsection{Alternating Optimization}

To solve the joint model onloading and offloading problem $P1$ in \eqref{eq:joint_opt}, we develop an alternating optimization algorithm, \textbf{J3O} (\textbf{J}oint \textbf{O}ptimization of \textbf{O}nloading and \textbf{O}ffloading). Alternating optimization (AO) is well suited for large-scale problems where variables decompose naturally into subproblems that are easier to solve individually than jointly \cite{chen2017joint, bi2020joint}. In our case, the problem is split into two interleaved stages: (i) model onloading, addressed via greedy submodular maximization with Lagrangian relaxation, and (ii) offloading, formulated as a constrained linear program. Each subproblem is solved while fixing the variables of the other, and the process iterates until convergence.

\subsubsection{Greedy Submodular Maximization via Lagrangian Relaxation}
The onloading subproblem maximizes a monotone submodular objective over binary variables $\mathbf{x}$ and $\mathbf{z}$, subject to memory and compute constraints. While greedy algorithms guarantee a $(1 - 1/e)$ approximation under a single knapsack constraint \cite{sviridenko2004note, iyer2013submodular}, our problem includes compute constraints that couple both $\mathbf{x}$ and $\mathbf{z}$. These interactions preclude direct use of standard greedy methods, motivating a Lagrangian relaxation to decouple the dependencies and enable tractable optimization.

\subsubsection{Greedy Submodular Maximization via Lagrangian Relaxation}
Let $\boldsymbol{\alpha} = ((\alpha^c)_c, (\alpha^e)_e)$ denote the Lagrange multipliers for the client and edge compute constraints. The corresponding dual objective is
\vspace{-0.03in}
\begin{equation}
\begin{aligned}
\mathcal{L}(&\mathbf{x}, \mathbf{z}; \, \boldsymbol\alpha) = F(\mathbf x, \mathbf z; \mathbf o)\\
& - \sum_{c \in \mathcal{C}} \alpha^c \left( \sum_{m,t} \bar\lambda_t^c (1 - o_t^c) w_m x_m^c z_{m,t}^c - \beta^c/\lambda_\text{tot} \right) \\
& - \sum_{e \in \mathcal{E}} \alpha^e \left( \sum_{m,t,c} \bar\lambda_t^c (o_t^c - o_t^{c,e}) w_m z_{m,t}^e - \beta^e/\lambda_\text{tot} \right).
\end{aligned}
\end{equation}
The model onloading step is then reformulated as a Lagrangian dual problem
\vspace{-0.09in}
\begin{align}
\min_{\boldsymbol\alpha}\; \max_{\mathbf{x}, \mathbf{z}}\,\mathcal{L}(\mathbf{x}, \mathbf{z}; \boldsymbol{\alpha}) \quad
\text{s.t.} \quad \eqref{const:assign-x-z}, \eqref{const:mem}. \notag
\vspace{-0.03in}
\end{align}
For fixed multipliers $\boldsymbol\alpha$, the inner maximization decouples across nodes (clients, edge), yielding a \textit{monotone submodular maximization} under a linear memory constraint. Each resulting subproblem can be efficiently approximated using a \textit{greedy algorithm} with provable approximation guarantees. Specifically, the greedy procedure selects models iteratively by maximizing marginal gain per unit memory ${\Delta \mathcal{L}^v(m; \mathcal{M}_{(i)}^v)}/{s_m}$, where $\Delta \mathcal L^v(m; \mathcal{M}_{(i)}^v) = \mathcal{L}^v(\mathcal{M}_{(i)}^v \cup {m}) - \mathcal{L}^v(\mathcal{M}_{(i)}^v)$ is the gain in Lagrangian objective for node $v \in \mathcal{V} = \mathcal{C} \cup \mathcal{E}$ when adding model $m$ to the current model set $\mathcal{M}_{(i)}^v$ at iteration $i$. The full procedure is formalized as Algorithm~\ref{alg:GreedyLR}.

After solving the inner problem via the greedy algorithm, the Lagrange multipliers are updated using a subgradient method \cite{fisher1981lagrangian} as
\vspace{-0.04in}
\begin{equation}
\begin{aligned}
\alpha^c &\leftarrow\!\!\left[ \alpha^c + \frac{\eta^c}{\sqrt{k}} \left( \sum_{m,t} \bar\lambda_t^c (1 - o_t^c) w_m x_m^c z_{m,t}^c - \frac{\beta^c}{\lambda_\text{tot}} \hspace{-0.03in}\right) \right]^+\hspace{-0.12in}, \\
\alpha^e &\leftarrow\!\!\left[ \alpha^e + \! \frac{\eta^e}{\sqrt{k}} \left( \sum_{m,t,c} \bar\lambda_t^c (o_t^c - o_t^{c,e}) w_m z_{m,t}^e - \frac{\beta^e}{{\lambda_\text{tot}}} \hspace{-0.03in}\right) \right]^+. \hspace{-0.12in}
\label{eq:dual-update}
\end{aligned}
\vspace{-0.03in}
\end{equation}
Here, $k$ is the iteration index, $\eta^c$ and $\eta^e$ are step size parameters, and $[\cdot]^+$ denotes projection onto the non-negative orthant. The updates proceed until constraint violations fall below a target threshold, yielding a near-feasible and near-optimal solution to the original problem.

\subsubsection{Offloading Optimization via Linear Programming} Given fixed onloading decisions $(\mathbf{x}, \mathbf{z})$ and the resulting accuracy matrices $A_t^c(\cdot, \cdot)$, $A_t^e(\cdot, \cdot)$ and $A_t^s$, we optimize the continuous offloading variables $\mathbf{o}$ by solving a constrained linear program. This can be efficiently handled using standard solvers such as Gurobi or CPLEX, yielding optimal offloading strategies across the hierarchy (client $\rightarrow$ edge $\rightarrow$ cloud).

\begin{algorithm}[t]
\SetAlgoNoEnd 
\DontPrintSemicolon
\footnotesize
  \KwIn{%
    Model pool $\mathcal{M}$, tolerance $\delta$; maximum iterations $N_{\max}$.}
  \KwOut{%
    Onloading decisions $\mathbf{x}^{\star},\mathbf{z}^{\star}$ and offloading rates $\mathbf{o}^{\star}$.}
  \textbf{Initialization:} 
  Greedily set $\mathbf{o}^{c, (0)}$ proportional to client task loads, under constraints; $\mathbf{o}^{c,e,(0)}\!\leftarrow\!\mathbf{0}$;\!
  $F(\mathbf{o}^{(0)};\mathbf{x}^{(0)}, \!\mathbf{z}^{(0)})\!=\!F_{\mathrm{best}}\!\!\leftarrow\!-\!\infty$.\;
  \For(\tcp*[f]{outer AO loop}){$k=1,\dots,N_{\max}$}{
    \textcolor{blue}{\textbf{/*  Greedy Lagrangian onloading  */}}\;
    Fix $\mathbf{o}^{(k-1)}$ and $\bigl(\mathbf{x}^{\mathrm{new}},\mathbf{z}^{\mathrm{new}}\bigr)\leftarrow$\textsc{Greedy-LR}$\,(\mathbf{o}^{(k-1)})$\;
    \If{$F(\mathbf{x}^{\mathrm{new}},\mathbf{z}^{\mathrm{new}};\mathbf{o}^{(k-1)})
         \hspace{-0.03in} > \hspace{-0.03in}F(\mathbf{o}^{(k-1)};\mathbf{x}^{(k-1)},\mathbf{z}^{(k-1)})$}{
         $\mathbf{x}^{(k)}\!\leftarrow\!\mathbf{x}^{\mathrm{new}},\;\mathbf{z}^{(k)}\!\leftarrow\!\mathbf{z}^{\mathrm{new}}$\;
    }\Else{
         $\mathbf{x}^{(k)}\!\leftarrow\!\mathbf{x}^{(k-1)},\;\mathbf{z}^{(k)}\!\leftarrow\!\mathbf{z}^{(k-1)}$\;
    }
    \textcolor{blue}{\textbf{/*  LP-based offloading  */}}\;
    Solve the linear program with fixed
    $\mathbf{x}^{(k)},\mathbf{z}^{(k)}$ to obtain $\mathbf{o}^{(k)}$.\;
    \If{$F\bigl(\mathbf{o}^{(k)};\mathbf{x}^{(k)},\mathbf{z}^{(k)}\bigr)-F_{\mathrm{best}}<\delta$}{
        \textbf{break}
    }
    $F_{\mathrm{best}}\leftarrow F\bigl(\mathbf{o}^{(k)};\mathbf{x}^{(k)},\mathbf{z}^{(k)}\bigr)$\;
  }
  \textbf{return}  $(\mathbf{x}^{\star}, \mathbf{z}^{\star}, \mathbf{o}^{\star}) \leftarrow(\mathbf{x}^{(k)},\mathbf{z}^{(k)}, \mathbf{o}^{(k)})$.
\caption{\textbf{J3O} Algorithm}
\label{alg:AO}
\end{algorithm}
\setlength{\textfloatsep}{5pt}

\begin{algorithm}[t]
\footnotesize
\SetAlgoNoEnd 
\DontPrintSemicolon
\KwIn{Fixed offloading rates $\mathbf{o}$; node sets
$\mathcal{V}=\mathcal{C}\cup\mathcal{E}$; tolerance $\varepsilon$.}
\KwOut{Binary onloading variables $(\mathbf{x},\mathbf{z})$}
\ForEach{$v\in\mathcal{V}$}{
  $\alpha^v\!\leftarrow\!0$,\,\,$\mathcal{M}^v_{(0)}\!\leftarrow\!\emptyset$ \tcp*{model set at node $v$}
}
\For(\tcp*[f]{Subgradient iteration}){$\ell = 1$ \KwTo $L_{\max}$}{
  \ForEach{$v\in\mathcal{V}$}{
    \While(\tcp*[f]{Greedy set selection}){True}{$m^\star\!\leftarrow\!\text{argmax}_{m\in\mathcal{M}\setminus\mathcal{M}^v_{(\ell)}}
          {\Delta \mathcal{L}^v (m;\mathcal{M}^v_{(\ell)}, \mathbf{o})}/{s_m}$\;
        \If{\textnormal{adding $m^\star$ keeps memory budget }$\mu^v$ feasible}{
            $\mathcal{M}^v_{(\ell)} \!\leftarrow\!
               \mathcal{M}^v_{(\ell)} \cup \{m^\star\}$\;
        }\Else{\textbf{break}}
    }
    Construct $(x^v, z^v)$ from ${\mathcal M^v_{(\ell)}}$\\ 
    $\alpha^v \leftarrow \bigl[\alpha^v + \eta_\ell \cdot\emph{viol}_v(\mathcal M^v_{(\ell)})\bigr]^+$ in \eqref{eq:dual-update}\\
    \If{$\emph{viol}_v(\mathcal M^v_{(\ell)})<\varepsilon$}{\textbf{break}}
  }

}
\Return{$(\mathbf{x},\mathbf{z})$}
\caption{\textsc{Greedy--LR}$(\mathbf{o})$}
\label{alg:GreedyLR}
\end{algorithm}

\subsubsection{Full Algorithm and Complexity Analysis}

Algorithm~\ref{alg:AO} summarizes the proposed J3O procedure. To ensure convergence and stability, we adopt a conservative update rule: the onloading decision at iteration $k$ is accepted only if it improves the overall objective relative to iteration $k-1$; otherwise, the previous decision is retained. This ensures the objective is non-decreasing over iterations, i.e.,
$F(\mathbf{x}^{(k)}, \mathbf{z}^{(k)}, \mathbf{o}^{(k)}) \geq F(\mathbf{x}^{(k-1)}, \mathbf{z}^{(k-1)}, \mathbf{o}^{(k-1)}), \; \forall k.$ The algorithm terminates once the improvement falls below a predefined threshold $\delta$, or a maximum number of iterations is reached. The final output is a near-optimal joint onloading–offloading configuration that balances accuracy and resource usage.

We analyze the runtime of J3O by focusing on the model onloading subproblem, as the LP-based offloading step is efficiently solvable. Let $K$ be the number of subgradient updates to the Lagrange multipliers. In each update, the greedy routine for each node $c \in \mathcal{C}$, $e \in \mathcal{E}$ evaluates the marginal gain of all $M$ candidate models at most $B^c$, $B^e$ times, respectively, bounded by memory budgets. Summing across nodes, the per-update cost is $\mathcal{O}\bigl(M(\sum_{c} B^c + \sum_{e} B^e)\bigr)$. Thus, one outer iteration has total cost $\mathcal{O}\bigl(KM(\sum_{c} B^c + \sum_{e} B^e)\bigr)$, which is linear in the number of models and memory budgets.

\vspace{-0.02in}
\subsection{Optimality Guarantee of Algorithm}
To establish a theoretical guarantee, we first analyze a simplified two-tier setting with clients and a central server; the extension to three tiers follows analogously.

\paragraph{Optimality Guarantee for Two-Level System}
We analyze the solution \((x^n, o^n)\) returned at iteration $n$ of our alternating optimization scheme for a two-level system comprising clients and a central server. The system objective is given by
\vspace{-0.05in}
\begin{equation}
F(x, o) = 1 - {\sum\nolimits}_{t} \,\,\bar\lambda_t (1 - o_t)(1 - A_t(x)),
\vspace{-0.01 in}
\end{equation}
where \(x\!\in\! \mathcal{X}\) denotes the binary model onloading decision under memory, and \(o\!\in\!\mathcal{O}(x)\) specifies the offloading rates subject to compute and communication budgets. The global optimum is denoted $(x^*, o^*)$.

\vspace{-0.03in}
\begin{assumption}[Monotone Onloading]
\label{ass:stop}
At each iteration $k=1, \dots, n$, let $x'$ be the onloading solution proposed by the greedy subroutine (given fixed $o^{k-1}$). The update is accepted only if it improves or maintains the objective, i.e., $F(x', o^{k-1}) - F(x^{k-1}, o^{k-1}) \ge 0$. Otherwise, the previous onloading decision is retained, i.e., $x^k = x^{k-1}$.
\end{assumption}

\begin{assumption}[Bounded Offloading Gap]
\label{ass:warmstart}
Let \(o^l\) denote the offloading solution obtained at iteration $l$ after convergence, and suppose $\|o^l - o^*\|_\infty \le \epsilon$.
Then the resulting objective is within \(\epsilon\) of the global optimum, which follows from
\begin{equation}
\begin{aligned}
F(x^*, o^*)- F(x^*, o^l)
& = {\sum\nolimits}_t \, \bar\lambda_t (o_t^* - o_t^l)(1 - A_t(x^*)) \\
& \le {\sum\nolimits}_t \, \bar\lambda_t \cdot |o_t^l - o_t^*| \\
& \le \max_t|o_t^l - o_t^*| \le \epsilon.
\end{aligned}
\end{equation}
\end{assumption}
\begin{lemma}[LP-Step Monotonicity]
\label{lem:lp}
Given a fixed onloading decision $x$, the offloading step, solved as an LP, yields a non-decreasing objective across iterations $k$. That is, $F(x, o^k) \ge F(x, o^{k-1})$, where $o^k = \arg\max_{o \in \mathcal{O}(x)} F(x, o)$.
\end{lemma}
\begin{proof}
By definition of the maximizer, $o^{k-1} \in \mathcal{O}(x)$ and $o^k$ is the optimal solution over the feasible set $\mathcal{O}(x)$. Therefore, $F(x, o^k) = \max_{o \in \mathcal{O}(x)} F(x, o) \ge F(x, o^{k-1})$.
\end{proof}

\begin{lemma}[Greedy Onloading with Lagrangian Relaxation]
\label{lem:greedy}
Suppose the onloading problem is solved via greedy submodular maximization under relaxed constraints. For fixed $o^k$, let $\alpha^* \ge 0$ be the Lagrange multiplier such that the greedy solution $x^{k+1}$ satisfies the relaxed constraint. Then, $F(x^{k+1}, o^k) \ge (1 - 1/e) \max_{x \in \mathcal{X}} F(x, o^k)$.
\end{lemma}
\begin{proof}
Let $\Delta_k := \beta - \phi(x^{k+1}, o^k) \ge 0$ be the slack in the relaxed constraint. Define the Lagrangian as $\mathcal{L}(x, o; \alpha) = F(x, o) - \alpha[\phi(x, o) - \beta]$. The greedy algorithm ensures
\begin{align}
\mathcal{L}(x^{k+1}, o^k; \alpha^*) \ge (1 - 1/e) \max_{x \in \mathcal{X}} \mathcal{L}(x, o^k; \alpha^*), \\
\max_{x \in \mathcal{X}}\,\mathcal{L}(x, o^k; \alpha^*) \ge\max_{x \in \mathcal{X}} F(x, o^k).
\end{align}
Since $F(x^{k+1}, o^k) = \mathcal L(x^{k+1}, o^k; \alpha^*) -\alpha^* \Delta_k$, we conclude
\begin{equation}
F(x^{k+1}, o^k) \ge (1 - 1/e) \max_{x \in \mathcal{X}} F(x, o^k) - \alpha^* \Delta_k.
\vspace{-0.03 in}
\end{equation}
Finally, the dual update $\alpha \leftarrow \max\{0, \alpha - \eta \Delta_k\}$ ensures $|\alpha^* \Delta_k| \le \varepsilon \approx 10^{-5}$, making the additive loss negligible.
\end{proof}

\begin{theorem}[Final Guarantee]
Under Assumptions~\ref{ass:stop} and~\ref{ass:warmstart}, the final iterate \((x^n, o^n)\) returned by the alternating optimization satisfies 
$F(x^n, o^n) \ge (1 - 1/e) [F(x^*, o^*) - \epsilon].$
\end{theorem}
\vspace{-0.1in}
\begin{proof}
By Lemma~\ref{lem:lp}, the objective is non-decreasing over LP steps, $F(x^n, o^n) \ge F(x^n, o^{n-1}) \ge \cdots \ge F(x^{l+1}, o^l).$
From Lemma~\ref{lem:greedy} and Assumption~\ref{ass:warmstart}, we have
\begin{equation}
\begin{aligned}
F(x^{l+1}, o^l) &\ge (1 - 1/e)
\max_{x} F(x, o^l) \\ &\ge (1 - 1/e) F(x^*, o^l)
 \\ &\ge (1 - 1/e)[F(x^*, o^*)-\epsilon],
\end{aligned}
\end{equation}
which completes the proof.
\end{proof}
\vspace{-0.1in}
For a system with clients, edges, and a central cloud with oracle accuracy, the offloading gap in Assumption~\ref{ass:warmstart} becomes
\begin{equation}
\begin{aligned}
    F(x^*, &o^{*})-F(x^{*}, o^k)\\&={\sum\nolimits}_t\, \bar\lambda_t^c(A_t^c(x^*)-A_t^e(x^*))(o_t^{c,k}-o_t^{c, *}) \\
    &+ {\sum\nolimits}_t \,\bar\lambda_t^c(1-A_t^e(x^*))(o_t^{c, e, k}-o_t^{c, e, *}).
\end{aligned}
\vspace{-0.03in}
\end{equation}
This shows that the total objective difference is again bounded by deviations in the offloading variables $o^c$ and $o^{c,e}$.

\section{Batching-Aware Extension}
\subsection{Batching Model for Edge Inference}

To capture the operational characteristics of edge accelerators, we extend our framework to incorporate batching, amortizing the per-query inference cost and introducing latency constraints tied to batch formation and execution. Each edge server collects incoming queries over a $T_b$
forming homogeneous batches, each with single-task queries. This design mirrors typical execution patterns in multi-task models, where inputs in the same batch are routed to shared task-specific components (e.g., classifier heads).

The resulting mean batch size for task $t$ at edge server $e$ is $b_t^e\!=\!\lambda_t^e T_b$. Empirical studies show batch execution latency grows linearly with batch size \cite{inoue2021queueing, cang2024joint}. We therefore model the total processing latency for task $t$ at the edge $e$ as
\vspace{-0.02in}
\begin{equation}
  \tau_t^e(b_t^e) =
      \sum_{m\in\mathcal M}\,
      \bigl(\nu_m^e \mathbbm{1}_{b_t^e} + \tau_m^e\,b_t^e\bigr)
      x_m^e z_{m,t}^e,
  \quad
  \tau_m^e = \frac{w_m}{\beta^e},
\vspace{-0.02in}
\end{equation}
where $\nu_m^e$ denotes the model-specific batch setup cost, $w_m$ is the per-query compute cost, and $\beta^e$ is the compute capacity of edge server $e$. Here, batching reduces per-query latency by amortizing the fixed $\nu_m^e$ over multiple queries within a batch. 

To ensure that on average, each batch initiated within an interval is executed before the subsequent interval begins, we impose a per-edge batching-latency constraint:
\vspace{-0.02in}
\begin{equation}
  \sum_{t\in\mathcal T}\sum_{m \in \mathcal M}\bigl(\nu_m^e \mathbbm{1}_{\lambda_t^e}+\frac{w_m}{\beta^e}\,\lambda_t^e T_b\bigr) x_m^e z_{m,t}^e
  \le T_b, \, \forall e, 
  \label{const:edge-capacity}
  \vspace{-0.03in}
\end{equation}
ensuring that the cumulative execution time of all task-specific batches launched on edge $e$ does not exceed the batching interval. This provides a conservative guarantee that the processing delay for each request does not exceed $2T_b$.

\subsection{Batching-Aware Joint Optimization}
We extend the J3O framework to incorporate batching-related latency constraints, resulting in the \textit{\textbf{B}atching-\textbf{A}ware \textbf{J}oint \textbf{O}ptimization of \textbf{O}nloading, \textbf{O}ffloading} (\textbf{BAJ3O}) algorithm.  A core challenge in this setting arises from the nonlinearity introduced by the indicator term $\mathbbm{1}_{\lambda_t^e}$, which captures whether task $t$ receives queries at edge $e$. This term is equivalent to the $\ell_0$-norm, $\|\lambda_t^e\|_0$, making the constraint non-convex. Following the surrogate approximation approach in \cite{cang2024joint}, we linearize this term using a first-order Taylor expansion, yielding the relaxed constraint
\vspace{-0.02in}
\begin{equation}
\!\!\!\!\sum_{t\in\mathcal T}\!\sum_{m \in \mathcal M}\!\bigg(\! \nu_m^e (\theta_t^e \lambda_t^e + \psi_t^e)\!+\!\frac{w_m}{\beta^e} \lambda_t^e T_b\!\bigg) x_m^e z_{m,t}^e \!\le\!T_b, \,\forall e, 
\label{const:relaxed-capacity}
\end{equation}
where $\theta_t^e$ and $\psi_t^e$ are iteration--specific coefficients derived from the surrogate function and updated to refine the approximation. Further details can be found in Section IV of \cite{cang2024joint}.

Leveraging this relaxation, BAJ3O performs alternating optimization over model onloading and offloading in the presence of batching. At each iteration, the batching constraint is first linearized using the current estimates of the surrogate coefficients $\theta_t^e$ and $\psi_t^e$. We then solve the onloading and offloading subproblems following the same procedure as in the original J3O algorithm, with the key difference that the edge compute constraint is now replaced by the batching-latency constraint. Once the model selections and offloading rates are updated, the surrogate coefficients are recalculated based on the updated effective task arrival rates $(\lambda_t^e)_{t\in\mathcal T, e\in\mathcal E}$,  reflecting the new batch sizes. This alternating process (consisting of constraint linearization, subproblem optimization, and surrogate update), repeats until convergance. 

\section{Simulation Results}

\begin{table}[t]
  \caption{Models and input statistics for each benchmark.}
  \label{tab:model_stats_combined}
  \centering
  \begin{subtable}[t]{\linewidth}
    \centering
    \scriptsize
    \vspace{-0.05in}
    \begin{tabular}{lcccc}
      \bottomrule
      \textbf{Model} & \textbf{\# Param} & \textbf{Mem (MB)} & \textbf{GFLOPs} 
      & \textbf{Input Mem (MB)}\\
      \hline
      Xception     & 18.41M & 73.66  & 9.17 & 0.79 \\
      ResNet34     & 21.80M & 83.15  & 3.68 & 0.60 \\
      \toprule
    \end{tabular}
    \caption{Taskonomy and DomainNet.}
  \end{subtable}
  \begin{subtable}[t]{\linewidth}
    \centering
    \vspace{-0.07in}
    \begin{adjustbox}{width=\linewidth}
    \begin{tabular}{lccccccc}
      \bottomrule
      \textbf{Metric} & $m_1$ & $m_2$ & $m_3$ & $m_{\{1,2\}}$ & $m_{\{1,3\}}$ & $m_{\{2,3\}}$ & $m_{\{1,2,3\}}$ \\
      \hline
      Mem (GB)-FP32  
        & 1.10 & 1.10 & 1.14 & 1.18 & 1.22 & 1.22 & 1.30\\
        Mem (GB)-INT8 & 0.80 & 0.80 & 0.84 & 0.88 & 0.91 & 0.91 & 0.99 \\
      TFLOPs &  1.13 & 1.13 & 0.73 & 1.84 & 1.44 & 1.44 & 2.16\\
      Input Mem (MB) & \multicolumn{7}{c}{25.17 (shared across models)}\\
      \toprule
    \end{tabular}
    \end{adjustbox}
    \caption{Cityscape3D. Numbers in braces denote supported tasks.}
  \end{subtable}
  \vspace{-0.13in}
\end{table}


\begin{table}[t]
  \caption{Specifications of client (C1–C5) and edge (E1–E3) devices across benchmarks. Task = Taskonomy, Dom = DomainNet, City = Cityscape3D.}
  \vspace{-0.05in}
\label{tab:device_config_both}
  \centering
  \footnotesize
  \begin{adjustbox}{width=\linewidth}
  \begin{tabular}{ll|ccccc|ccc}
    \bottomrule
    \textbf{Dataset} & \textbf{Metric} & \textbf{C1} & \textbf{C2} & \textbf{C3} & \textbf{C4} & \textbf{C5} & \textbf{E1} & \textbf{E2} & \textbf{E3} \\
    \hline
    \multirow{2}{*}{Task/Dom} 
      & Mem (MB)  & 24 & 24 & 48 & 48 & 96 & 512 & 512 & 1024 \\
    & GFLOPs    & 0.5K & 1.0K & 1.0K & 2.0K & 2.0K & 10.0K & 12.0K & 15.0K \\
    \hline
    \multirow{2}{*}{City}  & Mem (GB)  & 1.0 & 2.0 & 2.0 & 3.0 & 4.0 & 5.0 & 6.0 & 6.0 \\
      & TFLOPs    & 10.0 & 10.0 & 15.0 & 15.0 & 20.0 & 30.0 & 40.0 & 50.0 \\
    \toprule
  \end{tabular}
\end{adjustbox}
\end{table}

\subsection{Experimental Setup}
\subsubsection{Datasets and Models} 

\begin{figure*}[t]
 \centering
 \begin{minipage}[b]{0.324\textwidth}
   \centering
   \includegraphics[width=\linewidth]{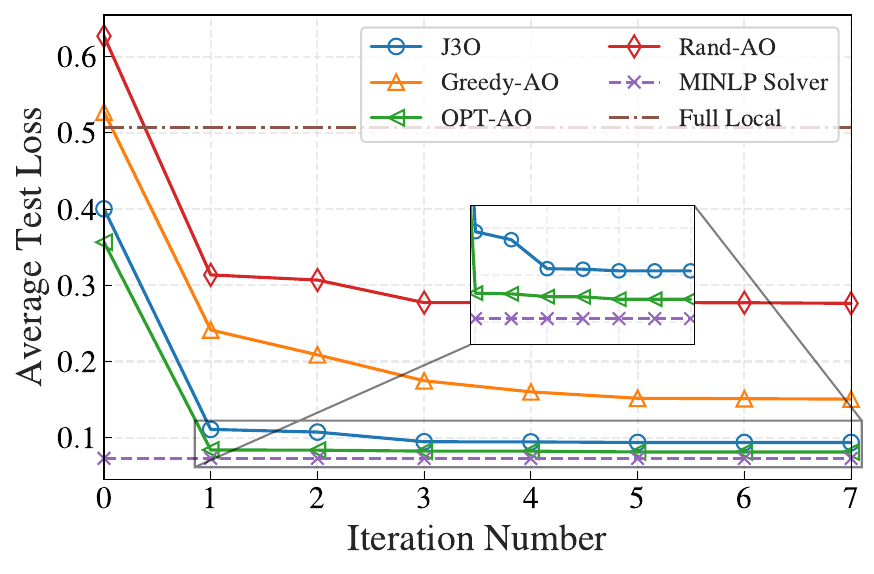}
   \vspace{-0.26in}
   \caption{Convergence of the algorithm.}
   \label{fig:converge}
 \end{minipage}  
 \begin{minipage}[b]{0.324\textwidth}
   \centering
   \includegraphics[width=\linewidth]{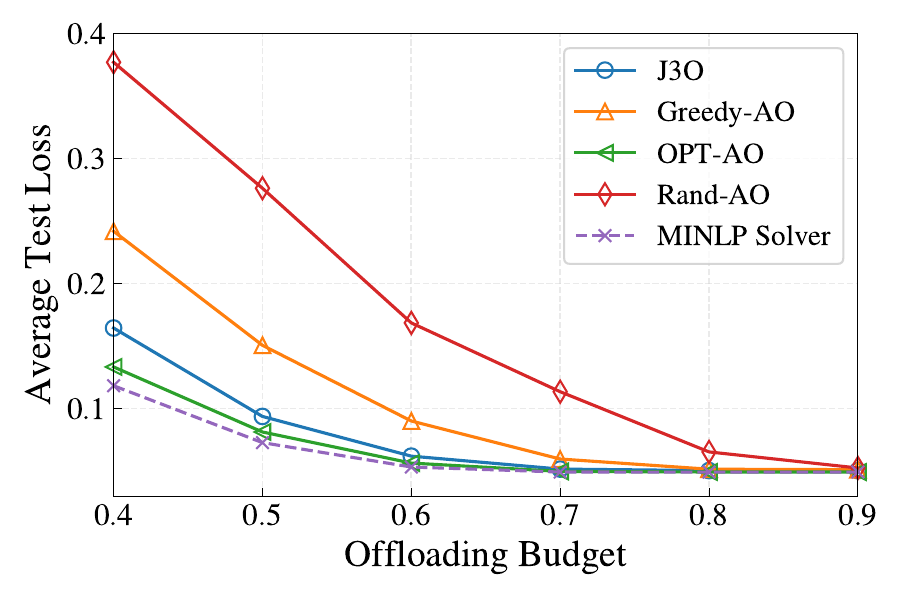}
   \vspace{-0.26in}
   \caption{Varying network capacity with $\chi_\kappa^e$.}
   \label{fig:network}
 \end{minipage}
 \begin{minipage}[b]{0.324\textwidth}
   \centering
   \includegraphics[width=\linewidth]{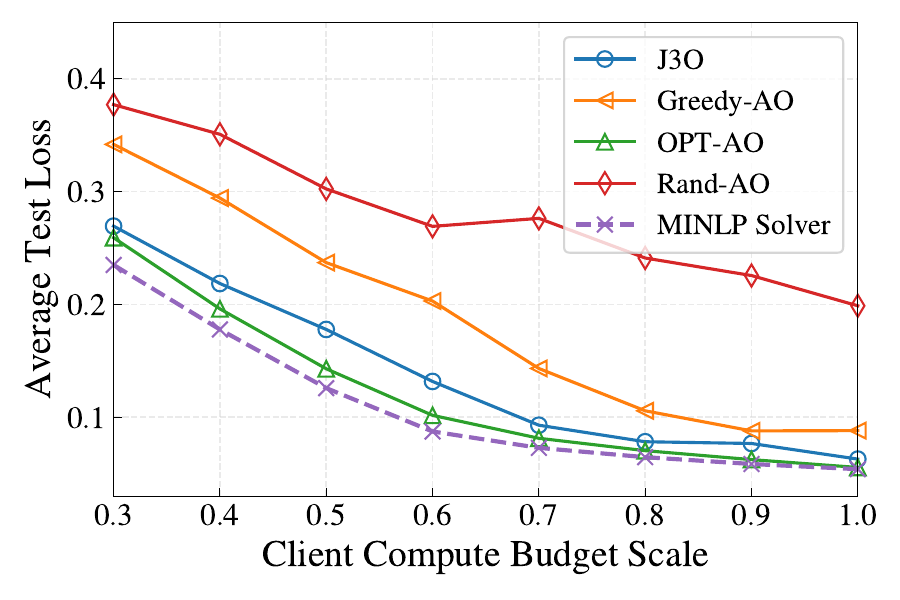}
   \vspace{-0.26in}
   \caption{Varying client compute budget.}
   \label{fig:compute}
 \end{minipage}
\end{figure*}


\begin{table*}[t]
\centering
\small
\vspace{-0.03in}
\caption{System performance (accuracy/loss) and runtime (ms), averaged over 20 seeds (±std).  Taskonomy/DomainNet use 31 models; CityScape3D uses 7.}
\vspace{-0.05 in}
\begin{tabular}{l|cc|cc|cc}
\bottomrule
\multirow{2}{*}{Method} &
\multicolumn{2}{c|}{\textbf{Taskonomy‑5}} &
\multicolumn{2}{c|}{\textbf{DomainNet‑6}} &
\multicolumn{2}{c}{\textbf{Cityscape3D‑3}}\\
& Avg.\ Loss ($\textcolor{red}{\boldsymbol\downarrow}$) & Runtime (ms) &
Avg.\ Acc ($\textcolor{blue}{\boldsymbol\uparrow}$) & Runtime (ms) &
Avg.\ Loss ($\textcolor{red}{\boldsymbol\downarrow}$) & Runtime (ms) \\
\hline
Greedy‑AO
& 0.151$_{\pm\text{0.081}}$
& 574.82$_{\pm\text{10.44}}$
& 0.735$_{\pm\text{0.025}}$
& 552.27$_{\pm\text{10.74}}$
& 0.164$_{\pm\text{0.014}}$
& 98.73$_{\pm\text{5.38}}$ \\

Optimal‑AO
& 0.081$_{\pm\text{0.047}}$
& 2078.33$_{\pm\text{369.35}}$
& 0.770$_{\pm\text{0.015}}$
& 2015.34$_{\pm\text{304.71}}$
& 0.163$_{\pm\text{0.010}}$
& 216.62$_{\pm\text{21.20}}$ \\

Rand‑AO
& 0.276$_{\pm\text{0.073}}$
& 562.67$_{\pm\text{10.53}}$
& 0.640$_{\pm\text{0.054}}$
& 537.89$_{\pm\text{13.92}}$
& 0.221$_{\pm\text{0.044}}$
& 95.81$_{\pm\text{5.16}}$ \\

Full‑Local
& 0.507$_{\pm\text{0.074}}$
& 633.71$_{\pm\text{244.23}}$
& 0.397$_{\pm\text{0.037}}$
& 597.49$_{\pm\text{237.37}}$
& 0.439$_{\pm\text{0.098}}$
& 40.22$_{\pm\text{4.34}}$ \\
\hline
MINLP Solver
& 0.073$_{\pm\text{0.037}}$
& 65791.68$_{\pm\text{21030.78}}$
& 0.774$_{\pm\text{0.013}}$
& 75122.31$_{\pm\text{18706.57}}$
& 0.159$_{\pm\text{0.008}}$
& 1123.47$_{\pm\text{303.61}}$ \\
\rowcolor{gray!15}
\textbf{J3O}
& 0.094$_{\pm\text{0.048}}$
& 797.87$_{\pm\text{8.43}}$
& 0.754$_{\pm\text{0.018}}$
& 937.08$_{\pm\text{29.79}}$
& 0.162$_{\pm\text{0.010}}$
& 165.10$_{\pm\text{12.02}}$ \\
\toprule
\end{tabular}
\label{tab:acc_runtime}
\vspace{-0.19 in}
\end{table*}

We evaluate our framework on three real-world multi-task benchmarks spanning diverse vision tasks and model architectures. \textbf{(1) Taskonomy-5} \cite{zamir2018taskonomy} comprises 5 indoor vision tasks. We evaluate performance using the task-wise losses reported in\cite{standley2020tasks}, with Xception backbones~\cite{chollet2017xception} and task-specific decoders. \textbf{(2) DomainNet-6} \cite{peng2019domainnet} is a multi-domain image classification benchmark covering 6 domains and 100 shared classes. Each domain is treated as a separate task. We fine-tune a shared ResNet-34 encoder with domain-specific heads\cite{wallingford2022task} and report test accuracy. \textbf{(3) Cityscape3D-3} \cite{gahlert2020cityscapes} is a 3-task benchmark for urban 3D perception, covering object detection, semantic segmentation, and depth estimation. We use TaskPrompter\cite{ye2022taskprompter} models, deploying full-precision variants on edge servers and INT8-quantized versions on clients. Performance is measured using task-specific test loss.


Table~\ref{tab:model_stats_combined} summarizes model configurations. In Taskonomy and DomainNet, models are executed in full precision by default, with INT8 client-side variants cached at 25\% memory cost. Due to hard parameter sharing~\cite{fifty2021efficiently,song2022efficient}, these backbones yield nearly uniform cost profiles across tasks. While this symmetry simplifies system behavior, our framework naturally extends to heterogeneous model profiles, as demonstrated in Cityscape3D. The cloud tier is assumed to host an oracle model with ideal task accuracy or loss.


\subsubsection{Device and System Parameters}
We study performance under heterogeneous hardware configurations in Table~\ref{tab:device_config_both}. Taskonomy and DomainNet represent smaller-scale deployments, with clients modeled after low-power devices (e.g., NVIDIA Jetson Nano or TX2) and edge servers with A10-class GPUs. In contrast, Cityscape3D reflects large-scale deployments with higher-end hardware, e.g., clients modeled after NVIDIA RTX 2080 or Mobile 500-class devices and edge servers with H200-class GPUs. Clients and edge servers are allocated 1–3 and 4–6 models, respectively, reflecting realistic storage constraints. This serves as the default configuration; we report results under varied resource budgets.

The simulated system consists of 30 clients and 3 edge servers, with each edge serving 10 clients. For each benchmark, we set the total task arrival rate to $\lambda_{\text{tot}}\!\in\!\{100, 2000, 4000\}$ jobs/sec and allocate this load across clients using a Dirichlet distribution with concentration parameter $p_{\text{client}}\! =\!0.5$. Each client distributes its individual load across tasks using a separate Dirichlet distribution with $p_{\text{task}}\!=\! 0.5$, resulting in task-specific rates $\lambda_t^c$. The offloading budget at each edge server $e$ is given by $\kappa^e\!=\!\chi_\kappa^e \sum_{t \in \mathcal{T}} \sum_{c \in \mathcal{C}e} \lambda_t^c d_t$, where $\chi_\kappa^e = 0.5$ (default), controlling the fraction of traffic that may be offloaded. The cross-edge (i.e., edge-to-cloud) budget is defined as $\kappa^s\!=\!\chi_\kappa^s \sum_e \lambda_t^e$ with $\chi_\kappa^s = 0.25$ by default.


\subsubsection{Baselines}
We compare our proposed J3O algorithm against the following five baselines:
\textbf{(1) MINLP Solver}: Solves the full joint problem optimally using Gurobi.
\textbf{(2) Greedy-AO}: Performs alternating optimization using greedy onloading (accepted only if the objective improves) and LP-based offloading.
\textbf{(3) OPT-AO}: Alternates between optimal onloading (via exhaustive search) and LP offloading.
\textbf{(4) Rand-AO}: Alternates with randomly selected feasible onloading configurations, accepted only if they improve the objective.
\textbf{(5) Full Local}: Executes all tasks using only client-side models, with optimal model selection per client.

\subsection{Performance and Runtime Comparison}
In Fig.~\ref{fig:converge}, our proposed algorithm J3O, evaluated on Taskonomy, shows monotonic objective improvement and converges within a few iterations. Table~\ref{tab:acc_runtime} reports average performance (accuracy/loss) and runtime across benchmarks. J3O consistently achieves near-optimal accuracy while significantly reducing runtime compared to the globally optimal MINLP solver. Specifically, J3O reaches 97.7\%, 97.5\%, and 99.7\% of the optimal performance on Taskonomy, DomainNet, and Cityscape3D, respectively, while using only 1.21\%, 1.25\%, and 14.70\% of the MINLP solver's runtime. While MINLP remains tractable on small problems like Cityscape3D (3 tasks, 7 models), its runtime scales poorly with problem size, making it impractical for real-world systems that require frequent, low-latency optimization. In contrast, J3O remains orders of magnitude faster on larger benchmarks.

Compared to the baselines, J3O offers clear advantages. Greedy-AO is faster due to conservative onloading but incurs a noticeable drop in performance. OPT-AO attains slightly higher accuracy on Taskonomy and DomainNet, but at the cost of $2.60\times$ and $2.15\times$ the runtime of J3O, respectively. On Cityscape3D, J3O outperforms OPT-AO on both accuracy and runtime. Both Rand-AO and Full-Local degrade significantly due to random model selection and lack of offloading, respectively. In contrast, J3O consistently balances accuracy and efficiency, yielding a robust trade-off.

\subsection{Robustness to Varying Resource Constraints}
We evaluate the robustness of J3O under varying system constraints and find that it consistently outperforms all non-optimal baselines. We consider two practical constraint types: (i) the offloading budget, which captures limitations in network bandwidth, and (ii) the client-side compute budget, reflecting the capabilities of resource-constrained devices. Fig.~\ref{fig:network} shows the average test loss ($\downarrow$) on the Taskonomy dataset as the offloading budget $(\kappa^e)_{e\in\mathcal E}$ is varied via the scaling parameter $\chi_\kappa^e$. While all methods improve with greater offloading capacity, J3O remains the closest to the optimal solution except for OPT-AO, which achieves slightly lower loss but at significantly higher runtime. As the offloading budget tightens, J3O’s advantage over heuristic baselines becomes more pronounced, demonstrating its robustness under network constraints. A similar trend appears in Fig.~\ref{fig:compute}, where we vary the client compute budgets $(\beta^c)_{c\in\mathcal C}$ via $\chi_\beta$. J3O consistently outperforms the heuristics and remains near-optimal across the full budget range, all while avoiding the computational overhead of exact solvers.



\begin{figure}[t]
 \centering
 \begin{minipage}[b]{0.24\textwidth}
   \centering   \includegraphics[width=1.0\linewidth]{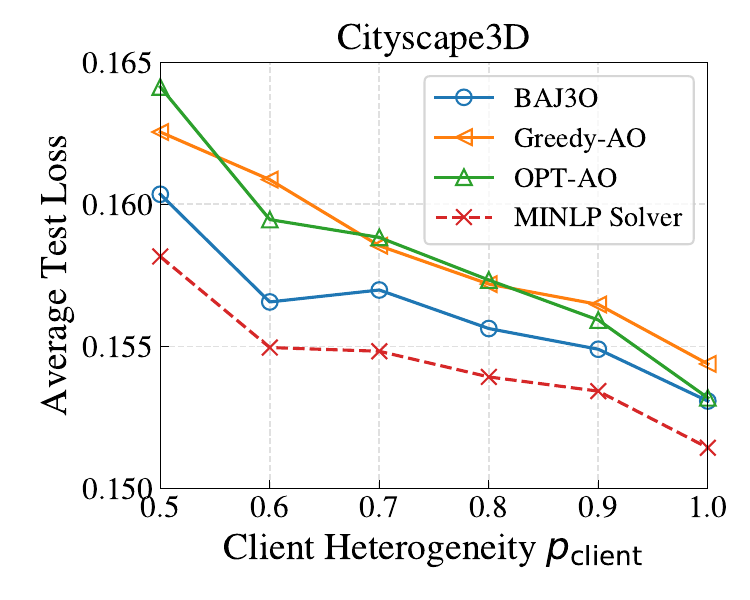}
   \vspace{-0.25in}
   \caption{Client heterogeneity.}
   \label{fig:batch-perf}
 \end{minipage}  
 \begin{minipage}[b]{0.24\textwidth}
   \centering
\includegraphics[width=1.0\linewidth]{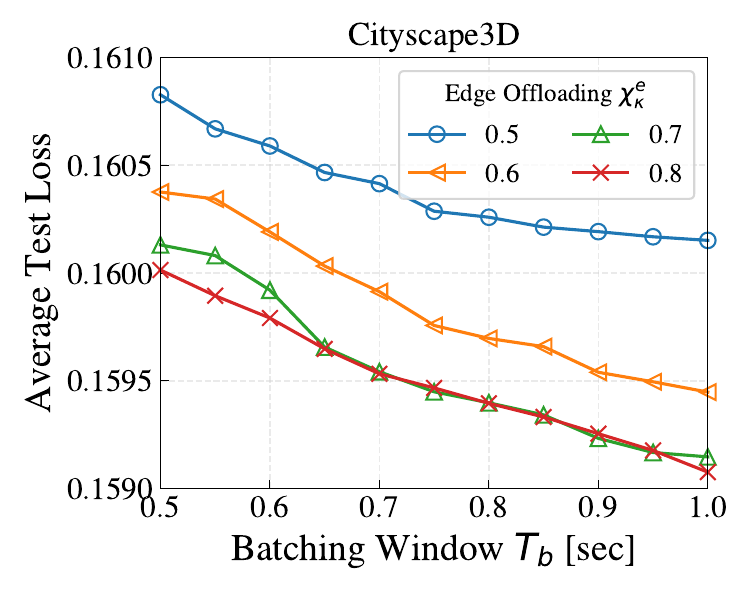}
   \vspace{-0.25in}
   \caption{Batching window.}
   \label{fig:batch-latency}
 \end{minipage}
\end{figure}

\subsection{Batching-Aware Joint Onloading and Offloading Results}

We evaluate BAJ3O on the Cityscape3D benchmark, estimating the batching setup cost $\nu_m^e$ via linear regression on measured latency profiles. The decision horizon is set to $10$s, with batching intervals of $T_b\!=\!0.5$s. As shown in Fig.~\ref{fig:batch-perf}, BAJ3O outperforms all baselines across varying client heterogeneity levels, achieving the smallest gap to the global optimum. Notably, BAJ3O surpasses OPT-AO, particularly under high heterogeneity. 
While both methods incorporate batching constraints during model onloading, OPT-AO optimally solves each subproblem based on a fixed batching estimate from the previous iteration. This fails to capture true batching dynamics under diverse client workloads, leading to biased early on/offloading decisions and reduced performance. 

In Fig.~\ref{fig:batch-latency}, we vary the batching interval $T_b$ to study its impact on latency and accuracy. As $T_b$ increases, BAJ3O forms larger batches, improving accuracy at the cost of higher delay, highlighting the accuracy–latency trade-off. This effect is stronger with generous edge offloading budgets, where batching is fully leveraged. Under tighter budgets ($\chi_\kappa^e = 0.5$), gains are limited as communication becomes the main bottleneck.

\section{Conclusion}
We presented a unified framework for joint model onloading and hierarchical offloading in distributed multi-task inference systems. By coordinating model placement and task routing across clients, edge servers, and the cloud, our J3O and BAJ3O algorithms enable accurate, efficient inference under tight resource constraints. Experiments on diverse benchmarks show our method achieves near-optimal accuracy with significantly reduced runtime. Our approach can be extended to settings where task demands and resource availability evolve slowly or load estimates are noisy. Indeed, our initial investigations indicate that this can be done efficiently, yielding excellent on/offloading decisions over time with minimal overhead.




\newpage
\bibliographystyle{IEEEtran}
\bibliography{infocom}

\begin{thebibliography}{10}
\providecommand{\url}[1]{#1}
\csname url@samestyle\endcsname
\providecommand{\newblock}{\relax}
\providecommand{\bibinfo}[2]{#2}
\providecommand{\BIBentrySTDinterwordspacing}{\spaceskip=0pt\relax}
\providecommand{\BIBentryALTinterwordstretchfactor}{4}
\providecommand{\BIBentryALTinterwordspacing}{\spaceskip=\fontdimen2\font plus
\BIBentryALTinterwordstretchfactor\fontdimen3\font minus \fontdimen4\font\relax}
\providecommand{\BIBforeignlanguage}[2]{{%
\expandafter\ifx\csname l@#1\endcsname\relax
\typeout{** WARNING: IEEEtran.bst: No hyphenation pattern has been}%
\typeout{** loaded for the language `#1'. Using the pattern for}%
\typeout{** the default language instead.}%
\else
\language=\csname l@#1\endcsname
\fi
#2}}
\providecommand{\BIBdecl}{\relax}
\BIBdecl

\bibitem{zhou2021device}
Q.~Zhou, Z.~Qu, S.~Guo, B.~Luo, J.~Guo, Z.~Xu, and R.~Akerkar, ``On-device learning systems for edge intelligence: A software and hardware synergy perspective,'' \emph{IEEE Internet of Things Journal}, vol.~8, no.~15, pp. 11\,916--11\,934, 2021.

\bibitem{villalobos2022machine}
P.~Villalobos, J.~Sevilla, T.~Besiroglu, L.~Heim, A.~Ho, and M.~Hobbhahn, ``Machine learning model sizes and the parameter gap,'' \emph{arXiv preprint arXiv:2207.02852}, 2022.

\bibitem{dhar2024empirical}
N.~Dhar, B.~Deng, D.~Lo, X.~Wu, L.~Zhao, and K.~Suo, ``An empirical analysis and resource footprint study of deploying large language models on edge devices,'' in \emph{Proceedings of the 2024 ACM Southeast Conference}, 2024, pp. 69--76.

\bibitem{Matsubara_2025_WACV}
Y.~Matsubara, M.~Mendula, and M.~Levorato, ``A multi-task supervised compression model for split computing,'' in \emph{Proceedings of the Winter Conference on Applications of Computer Vision (WACV)}, February 2025, pp. 4913--4922.

\bibitem{fan2022dnn}
W.~Fan, Z.~Chen, Z.~Hao, Y.~Su, F.~Wu, B.~Tang, and Y.~Liu, ``Dnn deployment, task offloading, and resource allocation for joint task inference in iiot,'' \emph{IEEE Transactions on Industrial Informatics}, vol.~19, no.~2, pp. 1634--1646, 2022.

\bibitem{fresa2022offloading}
A.~Fresa and J.~P.~V. Champati, ``An offloading algorithm for maximizing inference accuracy on edge device in an edge intelligence system,'' in \emph{Proceedings of the 25th International ACM Conference on Modeling Analysis and Simulation of Wireless and Mobile Systems}, 2022, pp. 15--23.

\bibitem{beytur2024HI}
H.~B. Beytur, A.~G. Aydin, G.~de~Veciana, and H.~Vikalo, ``Optimization of offloading policies for accuracy-delay tradeoffs in hierarchical inference,'' in \emph{IEEE INFOCOM 2024 - IEEE Conference on Computer Communications}, 2024, pp. 1989--1998.

\bibitem{ye2022taskprompter}
H.~Ye and D.~Xu, ``Taskprompter: Spatial-channel multi-task prompting for dense scene understanding,'' in \emph{The Eleventh International Conference on Learning Representations}, 2022.

\bibitem{neseem2023adamtl}
M.~Neseem, A.~Agiza, and S.~Reda, ``Adamtl: Adaptive input-dependent inference for efficient multi-task learning,'' in \emph{Proceedings of the IEEE/CVF Conference on Computer Vision and Pattern Recognition}, 2023, pp. 4730--4739.

\bibitem{doshi2022multi}
K.~Doshi and Y.~Yilmaz, ``Multi-task learning for video surveillance with limited data,'' in \emph{Proceedings of the IEEE/CVF Conference on Computer Vision and Pattern Recognition}, 2022, pp. 3889--3899.

\bibitem{standley2020tasks}
T.~Standley, A.~Zamir, D.~Chen, L.~Guibas, J.~Malik, and S.~Savarese, ``Which tasks should be learned together in multi-task learning?'' in \emph{International conference on machine learning}.\hskip 1em plus 0.5em minus 0.4em\relax PMLR, 2020, pp. 9120--9132.

\bibitem{fifty2021efficiently}
C.~Fifty, E.~Amid, Z.~Zhao, T.~Yu, R.~Anil, and C.~Finn, ``Efficiently identifying task groupings for multi-task learning,'' \emph{Advances in Neural Information Processing Systems}, vol.~34, pp. 27\,503--27\,516, 2021.

\bibitem{song2022efficient}
X.~Song, S.~Zheng, W.~Cao, J.~Yu, and J.~Bian, ``Efficient and effective multi-task grouping via meta learning on task combinations,'' \emph{Advances in Neural Information Processing Systems}, vol.~35, pp. 37\,647--37\,659, 2022.

\bibitem{li2018zalad}
H.~Li, C.~Hu, J.~Jiang, Z.~Wang, Y.~Wen, and W.~Zhu, ``Jalad: Joint accuracy-and latency-aware deep structure decoupling for edge-cloud execution,'' in \emph{2018 IEEE 24th International Conference on Parallel and Distributed Systems (ICPADS)}, 2018, pp. 671--678.

\bibitem{kang2017neurosurgeon}
Y.~Kang, J.~Hauswald, C.~Gao, A.~Rovinski, T.~Mudge, J.~Mars, and L.~Tang, ``Neurosurgeon: Collaborative intelligence between the cloud and mobile edge,'' \emph{ACM SIGARCH Computer Architecture News}, vol.~45, no.~1, pp. 615--629, 2017.

\bibitem{li2019edge}
E.~Li, L.~Zeng, Z.~Zhou, and X.~Chen, ``Edge ai: On-demand accelerating deep neural network inference via edge computing,'' \emph{IEEE Transactions on Wireless Communications}, vol.~19, no.~1, pp. 447--457, 2020.

\bibitem{he2020jpdra}
W.~He, S.~Guo, S.~Guo, X.~Qiu, and F.~Qi, ``Joint dnn partition deployment and resource allocation for delay-sensitive deep learning inference in iot,'' \emph{IEEE Internet of Things Journal}, vol.~7, no.~10, pp. 9241--9254, 2020.

\bibitem{hudson2021qos}
N.~Hudson, H.~Khamfroush, and D.~E. Lucani, ``Qos-aware placement of deep learning services on the edge with multiple service implementations,'' in \emph{2021 International Conference on Computer Communications and Networks (ICCCN)}, 2021, pp. 1--8.

\bibitem{sada2024energy}
A.~B. Sada, A.~Khelloufi, A.~Naouri, H.~Ning, and S.~Dhelim, ``Energy-aware selective inference task offloading for real-time edge computing applications,'' \emph{IEEE Access}, 2024.

\bibitem{zhang2021deep}
W.~Zhang, D.~Yang, H.~Peng, W.~Wu, W.~Quan, H.~Zhang, and X.~Shen, ``Deep reinforcement learning based resource management for dnn inference in industrial iot,'' \emph{IEEE Transactions on Vehicular Technology}, vol.~70, no.~8, pp. 7605--7618, 2021.

\bibitem{chai2024joint}
Y.~Chai, K.~Gao, G.~Zhang, L.~Lu, Q.~Li, and Y.~Zhang, ``Joint task offloading, resource allocation and model placement for ai as a service in 6g network,'' \emph{IEEE Transactions on Services Computing}, 2024.

\bibitem{chen2022online}
Z.~Chen, S.~Zhang, Z.~Ma, S.~Zhang, Z.~Qian, M.~Xiao, J.~Wu, and S.~Lu, ``An online approach for dnn model caching and processor allocation in edge computing,'' in \emph{2022 IEEE/ACM 30th International Symposium on Quality of Service (IWQoS)}.\hskip 1em plus 0.5em minus 0.4em\relax IEEE, 2022, pp. 1--10.

\bibitem{wu2024share}
Y.~Wu, J.~Wu, L.~Chen, B.~Liu, M.~Yao, and S.~K. Lam, ``Share-aware joint model deployment and task offloading for multi-task inference,'' \emph{IEEE Transactions on Intelligent Transportation Systems}, vol.~25, no.~6, pp. 5674--5687, 2024.

\bibitem{chen2017joint}
M.-H. Chen, B.~Liang, and M.~Dong, ``Joint offloading and resource allocation for computation and communication in mobile cloud with computing access point,'' in \emph{IEEE INFOCOM 2017-IEEE Conference on Computer Communications}.\hskip 1em plus 0.5em minus 0.4em\relax IEEE, 2017, pp. 1--9.

\bibitem{bi2020joint}
S.~Bi, L.~Huang, and Y.-J.~A. Zhang, ``Joint optimization of service caching placement and computation offloading in mobile edge computing systems,'' \emph{IEEE Transactions on Wireless Communications}, vol.~19, no.~7, pp. 4947--4963, 2020.

\bibitem{sviridenko2004note}
M.~Sviridenko, ``A note on maximizing a submodular set function subject to a knapsack constraint,'' \emph{Operations Research Letters}, vol.~32, no.~1, pp. 41--43, 2004.

\bibitem{iyer2013submodular}
R.~K. Iyer and J.~A. Bilmes, ``Submodular optimization with submodular cover and submodular knapsack constraints,'' \emph{Advances in neural information processing systems}, vol.~26, 2013.

\bibitem{fisher1981lagrangian}
M.~L. Fisher, ``The lagrangian relaxation method for solving integer programming problems,'' \emph{Management science}, vol.~27, no.~1, pp. 1--18, 1981.

\bibitem{inoue2021queueing}
Y.~Inoue, ``Queueing analysis of gpu-based inference servers with dynamic batching: A closed-form characterization,'' \emph{Performance Evaluation}, vol. 147, p. 102183, 2021.

\bibitem{cang2024joint}
Y.~Cang, M.~Chen, and K.~Huang, ``Joint batching and scheduling for high-throughput multiuser edge ai with asynchronous task arrivals,'' \emph{IEEE Transactions on Wireless Communications}, 2024.

\bibitem{zamir2018taskonomy}
A.~R. Zamir, A.~Sax, W.~Shen, L.~J. Guibas, J.~Malik, and S.~Savarese, ``Taskonomy: Disentangling task transfer learning,'' in \emph{Proceedings of the IEEE conference on computer vision and pattern recognition}, 2018, pp. 3712--3722.

\bibitem{chollet2017xception}
F.~Chollet, ``Xception: Deep learning with depthwise separable convolutions,'' in \emph{Proceedings of the IEEE conference on computer vision and pattern recognition}, 2017, pp. 1251--1258.

\bibitem{peng2019domainnet}
X.~Peng, Q.~Bai, X.~Xia, Z.~Huang, K.~Saenko, and B.~Wang, ``Moment matching for multi-source domain adaptation,'' in \emph{Proceedings of the IEEE/CVF international conference on computer vision}, 2019, pp. 1406--1415.

\bibitem{wallingford2022task}
M.~Wallingford, H.~Li, A.~Achille, A.~Ravichandran, C.~Fowlkes, R.~Bhotika, and S.~Soatto, ``Task adaptive parameter sharing for multi-task learning,'' in \emph{Proceedings of the IEEE/CVF Conference on Computer Vision and Pattern Recognition}, 2022, pp. 7561--7570.

\bibitem{gahlert2020cityscapes}
N.~G{\"a}hlert, N.~Jourdan, M.~Cordts, U.~Franke, and J.~Denzler, ``Cityscapes 3d: Dataset and benchmark for 9 dof vehicle detection,'' \emph{arXiv preprint arXiv:2006.07864}, 2020.

\end{thebibliography}

\end{document}